\documentclass[10pt]{article} 
\usepackage[preprint]{tmlr}

\usepackage{amsmath,amsfonts,amsthm,bm}

\newtheorem{theorem}{Theorem}
\newtheorem{lemma}{Lemma}

\def\eqref#1{equation~\ref{#1}}

\def\1{\bm{1}}

\def\rx{{\textnormal{x}}}
\def\ry{{\textnormal{y}}}

\def\rvx{{\mathbf{x}}}

\def\rvz{{\mathbf{z}}}

\def\vx{{\bm{x}}}

\def\vz{{\bm{z}}}

\DeclareMathAlphabet{\mathsfit}{\encodingdefault}{\sfdefault}{m}{sl}
\SetMathAlphabet{\mathsfit}{bold}{\encodingdefault}{\sfdefault}{bx}{n}

\def\gA{{\mathcal{A}}}

\def\gD{{\mathcal{D}}}
\def\gE{{\mathcal{E}}}
\def\gF{{\mathcal{F}}}

\def\gP{{\mathcal{P}}}

\def\gX{{\mathcal{X}}}
\def\gY{{\mathcal{Y}}}

\newcommand{\ind}{\mathbf{1}}

\newcommand{\E}{\mathop{\mathbb{E}}}

\newcommand{\R}{\mathbb{R}}

\DeclareMathOperator*{\argmax}{arg\,max}
\DeclareMathOperator*{\argmin}{arg\,min}

\usepackage{graphicx} 

\usepackage{hyperref}
\usepackage{url}
\usepackage{subcaption}
\usepackage{fontawesome} 

\newtheorem{corollary}{Corollary}
\newtheorem{definition}{Definition}
\newtheorem{proposition}{Proposition}
\newcommand{\keypoint}[1]{{\bf #1}\quad}

\makeatletter
\newcommand{\reviewhide}[1]{%
  \if@preprint
    #1
  \else
    [removed for review]
  \fi
}
\makeatother

\title{Strategic Classification with Randomised Classifiers}

\author{\name Jack Geary \email jack.geary@ed.ac.uk \\
      \addr School of Informatics\\
            University of Edinburgh\\
            Edinburgh\\
            United Kingdom
      \AND
      \name Henry Gouk \email henry.gouk@ed.ac.uk \\
      \addr School of Informatics\\
            University of Edinburgh\\
            Edinburgh\\
            United Kingdom}

\def\month{MM}  
\def\year{YYYY} 
\def\openreview{\url{https://openreview.net/forum?id=XXXX}} 

\begin{document}

\maketitle

\begin{abstract}
    Strategic classification typically assumes a Learner deploys a single deterministic classifier, giving Agents a fixed target classifier to game.
    We instead consider the setting where a Learner commits to a distribution over classifiers, and samples a classifier each time a prediction is made.
    This distribution is disclosed, so the Agents know the mixture but not which classifier will ultimately be applied to them.
    We show that the optimal randomised classifier can never be worse than the optimal deterministic one, and give sufficient conditions for arbitrary hypothesis classes and feature spaces under which the randomised approach is strictly better.
    We decompose this improvement into three mechanisms by which mixing classifiers deters gaming.
    Our empirical results on three tabular datasets show that two-component uniform mixtures of decision trees outperform the strategically optimal deterministic classifier for intermediate cost ranges.
    Previous work has considered the case where the Learner discloses partial or no information about the classifier to the Agents. Our results show that, if the Agents can aggregate observations, then randomisation preserves its robustness to strategic behaviour while deterministic classifiers do not. 
\end{abstract}

\begin{center}
    \faGithub \ \reviewhide{\href{github.com/Justme21/strategic-classification-framework}{Experiment Repo}}
\end{center}

\section{Introduction}
Classifiers built with machine learning can play a significant role in a number of resource allocation scenarios; universities determining which students to enrol for the coming year and banks deciding whether or not to give a customer a loan will rely on classifiers to determine the eligibility of candidates \citep{citron2014,milli2019}. In these settings, it is possible for candidates to use information about the classifier to strategically alter how they represent themselves to the system with the aim of improving their classification. This is known as ``gaming'' the classifier. Strategic Classification considers the problem of learning classifiers in the presence of such gaming behaviour.

Strategic Classification models an interaction between a \emph{Learner}, who chooses and publicly discloses a classifier, and \emph{Agents} who are subject to classification \citep{hardt2016}.\footnote{In the literature the Learner and Agent roles are also referred to as ``Jury'' and ``Contestant'', respectively \citep{hardt2016}.} The Agents are each independently motivated to be positively classified and, knowing the publicly disclosed classifier, are empowered to alter their representations in order to be classified favourably. The Learner's goal is to choose a classifier that achieves the highest classification accuracy possible, conditioned on this gaming behaviour. Most existing work in this area is restricted to the setting where the Learner must select a single classifier from a specified family of classifiers. This puts a heavy constraint on the Learner's options, and limits their ability to counteract the Agents' strategic behaviour.

We argue that the Learner should instead commit to a distribution over classifiers: each Agent can observe the distribution, but is classified by a classifier sampled independently from it. The optimal randomised classifier can never do worse than the optimal deterministic one, and can do strictly better, because Agents do not know which classifier they must game.

A related line of work instead limits what Agents know about a single classifier, by disclosing a set of candidates \citep{cohen2024} or nothing at all \citep{ghalme2021}. We find that, while concealing a fixed classifier can outperform randomisation when Agents act in isolation, this advantage is lost once Agents pool a small number of observations; the benefit of randomisation does not depend on concealment.

Our contributions are as follows:
\begin{itemize}
    \item We formulate strategic classification with randomised classifiers, show that the optimal randomised classifier is never worse than the optimal deterministic classifier, and characterise exactly when a uniform mixture of two classifiers outperforms the optimal deterministic classifier.
    \item We decompose the advantage of a randomised classifier into regions of the feature space corresponding to distinct mechanisms (``Dilution'', ``Deterrence'', ``Settling'') by which randomisation deters gaming. This yields sufficient conditions to achieve a positive advantage that hold for arbitrary hypothesis classes and measurable feature spaces.
    \item We show that in practice uniform mixtures over two decision trees can outperform the strategically optimal deterministic tree in an intermediate range of gaming costs, and show that increasing beyond two components provides littlr further benefit.
    \item We show that approaches which conceal a fixed classifier lose their advantage once Agents aggregate observations, whereas randomisation retains its advantage.
\end{itemize}

\section{Related Work}
\label{sec:related_work}
\keypoint{Strategic Classification} The literature in this area primarily builds upon the problem structure and nomenclature established by \cite{hardt2016}. However, earlier works such as \cite{dalvi2004} and \cite{bruckner2011} show that efforts to address the problem predate this. In their work, \citeauthor{hardt2016} established the convention of the Agent with some state that they will strategically manipulate, subject to cost constraints, in order to obtain a favourable classification from some publicly disclosed classifier deployed by the Learner. In the same work, \citeauthor{hardt2016} proposed an algorithm that could solve this problem, under the assumption of a separable cost function. Subsequent literature has produced formulations that do not rely on this property (e.g., \cite{levanon2021, miller2020, eilat2022}). Other works propose an alternative formulation which does not explicitly rely on the cost, $c$, but instead introduces the concept of a manipulation graph to define the set of feasible states \cite{zhang2021, lechner2022, lechner2023}. Common to all of these works is the use of a deterministic classifier. In contrast, we generalise the problem setting to allow for randomised classifiers and study the resulting implications.

\keypoint{Classifier Uncertainty} \citet{ghalme2021}, \citet{cohen2024}, and \citet{hikri2026} explore variants of the conventional Strategic Classification formulation where the classifier is presumed to be unknown to the Agents, and must be inferred. Both \citet{ghalme2021} and \citet{cohen2024} use distributions to capture the Agents' uncertainty about the ``true'' classifier; \citeauthor{cohen2024} model the Agents as maintaining a belief over possible classifier definitions. The Learner can then shape the information they reveal about the classifier to the Agents in order to control their ability to game, with the goal of maximising accuracy. \citeauthor{ghalme2021} instead explore the case where the classifier is not revealed to the Agents, and so they have to approximate it from observation data about the classifier's behaviour. The authors demonstrate that, under certain assumptions, not revealing classifier definition can result in considerable accuracy losses for the Learner. \citeauthor{hikri2026} consider an alternative problem formulation, where the Learner discloses an ``ambiguity set'' containing the true classifier, and Agents game only if they can achieve a positive classification from every member of the set. In these works, a single classifier is deployed and distributions capture the Agents' uncertainty over the classifier chosen by the Learner. In contrast, in this work the distribution is what is chosen by the Learner, so the Agents' uncertainty cannot be resolved by observing its decisions.

\keypoint{Randomised Classifiers} Prior work has not provided a general investigation into the idea of learning randomised classifiers for strategic settings; \citet{braverman2020} investigate the behaviour of randomised linear classifiers for one dimensional real-valued feature spaces---i.e., threshold functions. \citet{sundaram2023} provide an example of a distribution defined on $\R^2$ where a specific randomised linear classifier will outperform the optimal deterministic classifier. However, the remainder of their work is on the sample complexity of learning algorithms that produce deterministic models. Our Theorem \ref{thm:mixture_vs_deterministic} and Corollary \ref{corr:mixture_vs_deterministic_optimal} can be seen as substantial generalisations of the claims about randomisation made in these works. In contrast to \citet{braverman2020} and \citet{sundaram2023}, our result applies to arbitrary hypothesis classes (not just linear models), and does not rely on constructing specific data distributions or specific Euclidean spaces; we identify a small set of sufficient conditions and allow features to come from any measurable space. Moreover, our results can be seen as a generalisation of the work of \cite{pinot2020} beyond a zero-sum adversarial robustness setting.


\section{Strategic Classification with Randomisation}
\label{sec:strategic_classification}

Throughout, $\gP(\gA)$ denotes the set of probability measures over a measurable space $\gA$. Let $\gD \in \gP(\gX \times \gY)$ be a data distribution over a feature space $\gX$ and labels $\gY=\{-1,1\}$, let $\gF$ be a family of classifiers mapping $\gX$ to $\gY$, and let $l(\hat \ry, \ry) = \ind[\hat \ry \neq \ry]$ be the zero–one loss.

The strategic classification problem \citep{hardt2016} differs from the conventional i.i.d. learning setting in that the distribution of data used to select a classifier from $\mathcal{F}$ is different to the distribution encountered at test time. In particular, at test time it is assumed that Agents with knowledge of the chosen classifier will strategically modify features according to some cost model in order to obtain positive classifications. This interaction is modelled as a Stackelberg Game between a Learner player and an unknown number of Agent players, with the Learner as the leader \citep{stackelberg1934}. The Learner player chooses a classifier, $f$, to classify the Agents. The Agents observe $f$ and, in response, attempt to ``game'' the classifier by independently perturbing their features, $\Delta_{f}(\rvx)$, with the aim of being classified as the positive class. Concretely, the Agents optimise a utility,
\begin{equation}
    \Delta_{f}(\vx) \in \text{BR}(f) := \argmax_{\vz \in \gX} f(\vz) - c(\vx,\vz),
    \label{eqn:strategic_response}
\end{equation}
where $\text{BR}(f)$ denotes the set of functions that act as best responses to $f$ according to the Agent, and $c: \mathcal{X} \times \mathcal{X} \rightarrow \mathbb{R}^{+}$ is a non-negative function quantifying the cost incurred by the Agent to alter their features. As is typical in the literature, we assume the positive classification is the desired outcome for all Agents and that all Agents use the same cost function, which is assumed to be known to the Learner. Agents are modelled as being rational; if the Agent is already positively classified ($f(\vx)=1$), then $\Delta_{f}(\vx) = \vx$. Further, if gaming does not improve the Agent's utility, then they do not game. Finally, if the Agent has multiple equivalent gaming actions, they choose the cheapest action. 

As in the standard i.i.d learning problem, the goal is to identify a classifier, $f \in \gF$, that minimises the strategic risk over an unknown data distribution, $\gD$. Given the Agents' gaming strategy, $\Delta$, the strategic risk is defined as
\begin{equation}
    R_{\Delta}(f) = \E_{(\rvx,\ry) \sim \mathcal{D}}[l(f(\Delta(\rvx)), \ry)],
    \label{eqn:strategic_risk}
\end{equation}
and the empirical strategic risk on the training set, $S$, is given by
\begin{equation}
    r_{\Delta}(f) = \frac{1}{n}\sum_{i=1}^{n}l(f(\Delta(\rvx_{i})),\ry_{i}).
    \label{eqn:empirical_strategic_risk}
\end{equation}
The idealised objective for the Learner is therefore to solve a bi-level optimisation problem,
\begin{equation}
    f^\ast = \argmin_{f \in \gF} R_{\Delta_f} (f),
\end{equation}
where the lower level of the problem arises from the definition of $\Delta_f$. Conventional approaches to this problem approximate the solution of this via a variant of empirical risk minimisation that takes into account the bi-level structure of the optimisation problem \citep{hardt2016, levanon2021, levanon2022, geary2026}. This idea has become known as Strategic Empirical Risk Minimisation (SERM) \citep{levanon2021}, and we denote the model obtained via this method by
\begin{equation}
\label{eqn:serm}
    \hat{f} = \argmin_{f \in \gF} r_{\Delta_f}(f).
\end{equation}

\subsection{Generalising to Randomised Classifiers}
\label{sec:mixed_strategies_and_cogameability}
In the conventional strategic classification problem formulation, the Learner commits to using a single classifier from $\gF$ to make all predictions at test time. We propose that the Learner instead commit to a distribution over classifiers, $Q \in \gP(\gF)$. $Q$ is treated as a Gibbs classifier (\citet{ng2001convergence}); when classifying each Agent's features at test time, the Learner samples a classifier according to this distribution and then uses it to make a prediction. Crucially, a new classifier will be sampled each time a prediction is to be made.

As a result of the uncertainty in the classifier behaviour introduced by this randomisation, the Agents' objective is revised to optimise the expected utility\footnote{See, e.g., \citet{berger2013} or \citet{maschler2020} for discussions on why this is justified.},
\begin{equation}
    \Delta_Q(\vx) = \argmax_{\vz \in \gX} \E_{f \sim Q}[f(\vz)] - c(\vx, \vz).
    \label{eqn:randomised_best_response}
\end{equation}
The strategic risk and its empirical counterpart are therefore generalised to
\begin{equation}
    R_\Delta(Q) = \E_{f \sim Q} \E_{(\rvx, \ry) \sim D}[l(f(\Delta(\rvx)), \ry)]
    \label{eqn:randomised_strategic_risk}
\end{equation}
and
\begin{equation}
    r_\Delta(Q) = \E_{f \sim Q} \left[ \frac{1}{n} \sum_{i=1}^n l(f(\Delta(\rvx_i)), \ry_i) \right],
\end{equation}
respectively, and the optimal randomised classifier, $Q^{\ast}$, is obtained by solving
\begin{equation}
    Q^{\ast} = \argmin_{Q \in \gP(\gF)} R_{\Delta_{Q}}(Q).
\end{equation}
Similar to the deterministic case, we can also define the SERM solution for the randomised classifier setting,
\begin{equation}
\label{eqn:randomised_serm}
    \hat{Q} = \argmin_{Q \in \gP(\gF)} r_{\Delta_Q}(Q).
\end{equation}
Observe that the delta distribution, which concentrates all probability mass on a single classifier in $\gF$ is contained within $\gP(\gF)$. As a consequence the formulation in Equation \ref{eqn:randomised_serm} is a strict generalisation of the conventional strategic learning problem (Equation \ref{eqn:serm}). Therefore the optimal randomised classifier can never perform worse than the optimal deterministic classifier.


\section{When Randomisation Works}
\label{sec:conditions_for_randomisation}
We are interested in diagnosing the conditions under which the optimal randomised classifier, $Q^\ast$, achieves a strictly lower strategic risk than the optimal deterministic classifier, $f^\ast$. Figure \ref{fig:three_balls_example} provides a toy example demonstrating how randomised classifiers can affect optimal strategic behaviour. Figure \ref{fig:three_balls_example} (left) presents a cluster of negatively labelled points with positively labelled points evenly distributed on either side. Both classes have equal probability mass ($P(y=1) = P(y=-1)$). If $\gF$ is the set of threshold classifiers on the x-axis, the optimal behaviour puts the classifier decision boundary between the negative cluster and one of the positive clusters. Regardless of which clusters it lies between, this classifier can be gamed by all the negative points (Figure \ref{fig:three_balls_example} (middle)). As a result the classifier will misclassify the majority of points in one of the positive regions, and the entirety of the negative region, resulting in a strategic risk of $0.61$.

In contrast, Figure \ref{fig:three_balls_example} (right) depicts the result from applying a uniform-mixture classifier over the two optimal threshold classifiers in the same setting; as the negative points cannot game both classifier components, they are forced to choose one to game (this is the Dilution mechanism of Section \ref{sec:mechanisms}), getting misclassified by the mixture with probability $0.5$. As a result, in expectation half of the points in each class will be misclassified, resulting in a strategic risk of $0.5$, outperforming the optimal deterministic behaviour.

We compare the average strategic risk of a pair of deterministic classifiers, $f,f' \in \gF$, with the strategic risk achieved from a uniform mixture, $Q = U(\{f,f'\})$. We denote the scaled difference between these quantities by $\Lambda(f, f')$,
\begin{align}
    \label{eqn:lambda_def}
    \Lambda(f,f') &:= 2[\E_{g \sim Q}[R_{\Delta_{g}}(g)] - R_{\Delta_{Q}}(Q)] \nonumber\\
    &= R_{\Delta_{f}}(f)+ R_{\Delta_{f'}}(f') - 2R_{\Delta_{Q}}(Q),
\end{align}
where multiplication by two is included to simplify notation later. $\Lambda$ captures the change in strategic risk achieved by pursuing a uniform mixed strategy over classifiers $f,f' \in \gF$ instead of pursuing either classifier individually as a deterministic strategy.

Let $R_{\Delta}^{*} := \min_{g \in \gF} R_{\Delta_{g}}(g)$ be the optimal strategic risk from a deterministic classifier, and the excess strategic risk associated with a classifier $g \in \gF$ be given by
\begin{equation}
    \gE(g) := R_{\Delta_{g}}(g) - R_{\Delta}^{*} \geq 0.
\end{equation}
With this established we can prove the following condition under which the optimal mixture outperforms the optimal deterministic classifier; $R_{\Delta_{Q^{*}}}(Q^{*}) < R_{\Delta}^{*}$.

\begin{theorem}
    \label{thm:mixture_vs_deterministic}
    Let $f,f' \in \mathcal{F}, Q=U(\{f,f'\})$. Then $R_{\Delta_{Q}}(Q) < R_{\Delta}^{*}$ if and only if $\Lambda(f,f') > \gE(f) + \gE(f')$. In this case $R_{\Delta_{Q^{*}}}(Q^{*}) \leq R_{\Delta_{Q}}(Q) < R_{\Delta}^{*}$.
\end{theorem}

\begin{proof}
    By definition $R_{\Delta_{g}}(g) = R_{\Delta}^{*} + \gE(g)$. Substituting this into Equation \ref{eqn:lambda_def} gives
    \begin{align}
        &\Lambda(f,f') = 2R_{\Delta}^{*} + \gE(f) + \gE(f') - 2R_{\Delta_{Q}}(Q).\\
        & \implies R_{\Delta_{Q}}(Q) = R_{\Delta}^{*} + \dfrac{1}{2}\left[\gE(f) + \gE(f') - \Lambda(f,f')\right].
    \end{align}
    If $\Lambda(f,f') > \gE(f) + \gE(f')$ then, by definition, $R_{\Delta_{Q^{*}}}(Q^{*}) \leq R_{\Delta_{Q}}(Q) < R_{\Delta}^{*}$ as $Q^{*}$ minimises the strategic risk over $\gP(\gF)$.
\end{proof}

\begin{figure}[t]
    \centering
    \includegraphics[width=\linewidth]{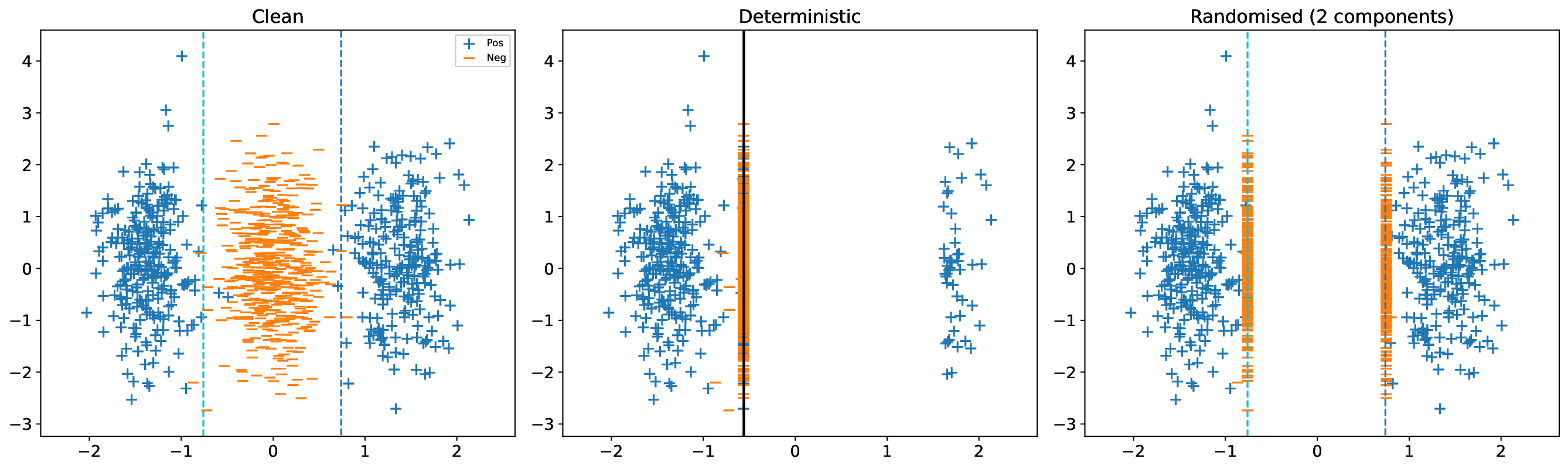}
    \caption{Depicting Threshold Classifiers applied to the three balls synthetic dataset: (left) Clean data distribution; (middle) Strategically optimal deterministic classifier; (right) Strategically optimal uniform mixture.}
    \label{fig:three_balls_example}
\end{figure}

This result can be easily extended to the set of strategically optimal classifiers $\gF^{*} = \argmin_{f \in \gF} R_{\Delta_{f}}(f)$ as follows:

\begin{corollary}
\label{corr:mixture_vs_deterministic_optimal}
    If $f,f' \in \gF^{*}$ and $Q=U(\{f,f'\})$ then $R_{\Delta_{Q^{*}}}(Q^{*}) \leq R_{\Delta_{Q}}(Q) < R_{\Delta_{f^{*}}}(f^{*})$ if and only if $\Lambda(f,f')>0$.
\end{corollary}

\begin{proof}
    This follows directly from Theorem \ref{thm:mixture_vs_deterministic} with the observation that $f,f' \in \gF^{*} \implies \gE(f) = \gE(f') = 0$. 
\end{proof}

\begin{figure}[t]
    \centering
    \includegraphics[width=\linewidth]{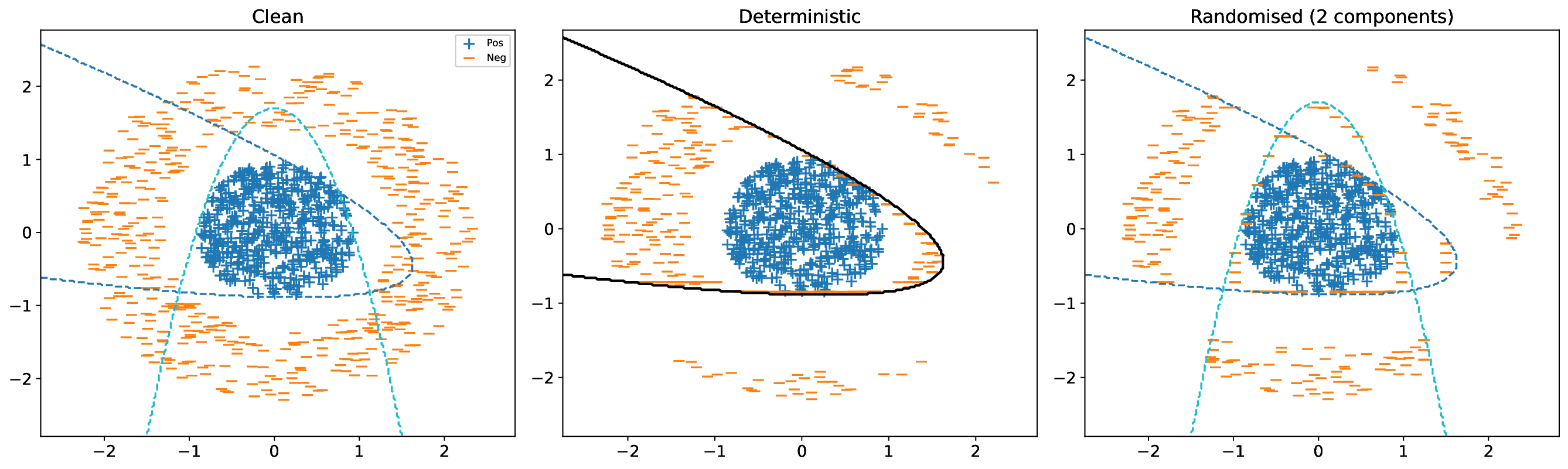}
    \caption{Depicting Parabolic Classifiers applied to the ball and ring synthetic dataset; (left) Clean data distribution; (middle) Strategically optimal deterministic classifier; (right) Strategically optimal uniform mixture.}
    \label{fig:ball_and_ring_example}
\end{figure}

Figure \ref{fig:ball_and_ring_example} considers the setting proposed in Corollary \ref{corr:mixture_vs_deterministic_optimal} where every classifier in the hypothesis class is optimal; $f,f' \in \gF^{*}$. Figure \ref{fig:ball_and_ring_example}(left) depicts the clean distribution; a ball of positively labelled points centred symmetrically at the origin surrounded by an equal number of negatively labelled points uniformly distributed around a disk. The hypothesis class, $\gF$, is the set of parabolic classifiers with a fixed offset width rotated about the origin. By symmetry every classifier has equal strategic risk, so $\gF = \gF^{*}$ and Corollary \ref{corr:mixture_vs_deterministic_optimal} applies.

Figure \ref{fig:ball_and_ring_example}(middle) demonstrates that an optimal deterministic classifier is very vulnerable to strategic behaviour, with most of the negatively classified points able to attack the classifier; $r_{\Delta_{\hat{f}}}(\hat{f}) = 0.43$. In the case of the mixture, Figure \ref{fig:ball_and_ring_example}(right),  while the components have the same strategic risk, the optimal decision boundaries are oriented such that it maximises the size of the region where negatively labelled points are misclassified by one component, but it is too expensive to game the other component. As such the mixture risk benefits from the relative orientation of the components. Supporting the result of the corollary; $\Lambda(f,f') = 0.166>0$ and $r_{\Delta_{\hat{Q}}}(\hat{Q}) = 0.34< r_{\Delta_{\hat{f}}}(\hat{f})$. Here the improvement comes from Deterrence (Section \ref{sec:mechanisms}).

Theorem \ref{thm:mixture_vs_deterministic} reduces the question of when mixing helps to the magnitude of $\Lambda(f,f')$. Since the value of $\Lambda$ depends on the best response of the mixture, $\Delta_{Q}$, the theorem does not characterise when this condition holds. Section \ref{sec:mechanisms} addresses this question. 

\subsection{Why Mixing Deters Gaming}
\label{sec:mechanisms}
To understand when $\Lambda$ is large, we decompose it into probability masses over regions of the feature space, each inducing distinct gaming behaviour. We provide an interpretable explanation of this definition by exploring the gaming behaviours induced in each region.

Let $f, f' \in \gF$ and $c : \gX \times \gX \to \mathbb{R}^{+}$ be the cost of strategic behaviour, as in Equation \ref{eqn:strategic_response}. The costs associated with the possible rational gaming behaviours can be derived as follows;
\begin{align}
\label{eqn:jpq_def}
    j(\vx) &:= \inf\{c(\vx, \vz): f(\vz)=f'(\vz)=1\} \nonumber\\
    p(\vx) &:= \inf \{c(\vx, \vz): f(\vz)=1, f'(\vz)=-1\}\\
    q(\vx) &:= \inf \{c(\vx, \vz): f(\vz)=-1, f'(\vz)=1\} \nonumber.
\end{align}
Succinctly, $j(\vx)$ is the minimum cost required to achieve a positive classification from both $f$ and $f'$ simultaneously. $p(\vx)$ is the minimum cost to achieve a positive classification from $f$ and a negative classification from $f'$. $q(\vx)$, symmetrically, is the minimum cost to achieve a negative classification from $f$ and a positive classification from $f'$.

Next define $\beta(A)$, for some measurable set $A \subseteq \gX$, as the signed class bias measure,
\begin{equation}
\label{eqn:beta_def}
    \beta(A) := P(\rx \in A, \ry=-1) - P(\rx \in A, \ry = 1).
\end{equation} 

For $Q= U(\{f,f'\})$, these are sufficient to decompose the input space, $\gX$, into regions that induce distinct gaming behaviours in $\Delta_{Q}$. This decomposition is given by the following theorem;
\begin{theorem}
\label{thm:lambda_def}
    For any $f,f' \in \gF, Q = U(\{f,f'\})$,
    \begin{align}
        \Lambda(f,f') = \beta(A) &+ 2\beta(B) + \beta(C_{f}^{\text{prot}})+ \beta(C_{f'}^{\text{prot}}) + \beta(W_{f}) + \beta(W_{f'})
    \end{align} 
    where,
    \begin{align*}
        &A := \{\rx \in \gX: f(\rx) = f'(\rx) = -1, \min(p,q)<1, \min(p,q)<j-1, (j<2 \text{ or } \max(p,q)<2)\}, \\
        &B := \{\rx \in \gX: f(\rx) = f'(\rx) = -1, j \geq 2, p \in [1,2), q \in [1,2)\},\\
        &C_{f}^{\text{prot}} := \{\rx \in \gX: f(\rx)=f'(\rx)=-1, j \geq 2, p \in [1,2), q\geq 2\}, \\
        &C_{f'}^{\text{prot}} := \{\rx \in \gX: f(\rx)=f'(\rx)=-1, j \geq 2, q \in [1,2), p\geq 2\}, \\
        &W_{f} = \{\rx \in \gX: f(\rx)=1, f'(\rx)=-1, j \geq 1, \min(j, q)<2\},  \\
        &W_{f'} = \{\rx \in \gX: f(\rx)=-1, f'(\rx)=1, j \geq 1, \min(j, p)<2\}. \\
    \end{align*}
\end{theorem}
In the interests of space, the proof of this theorem is deferred to Appendix \ref{app:lambda_def_proof}.

Each region defined in Theorem~\ref{thm:lambda_def} identifies a set of points where Agents behave differently when responding to a mixture ($\Delta_{Q}$) as compared with responding to either component individually ($\Delta_{f}$ or $\Delta_{f'}$). $\beta$ measures the negative class bias within each region. The regions fall into three groups, each corresponding to a distinct mechanism (Appendix \ref{app:gaming_region_interpretation} elaborates on the behaviours in each region).
\begin{itemize}
    \item \textbf{Dilution} ($A$) Points that are rejected by both mixture components, where either component could be gamed individually, but are too expensive to be gamed jointly. Agents are forced to choose one component to game, therefore accepted with probability $\frac{1}{2}$ rather than $1$. Figure~\ref{fig:three_balls_example} is an instance: no point is accepted by both thresholds, so a negative Agent can only ever game one of them.
    \item \textbf{Deterrence} ($B$, $C_{f}^{\text{prot}}, C_{f'}^{\text{prot}}$) Points where the Agents are rejected by both components, but where each component is worth gaming individually. Specifically this relates to regions where it is not worth gaming either component if it yields acceptance only with probability $\frac{1}{2}$. Facing the mixture they do not game at all. Figure~\ref{fig:ball_and_ring_example} contains instances of this behaviour: the components are oriented so that some negative Agents who could afford to cross one boundary cannot afford to cross both.
    \item \textbf{Settling} ($W_{f}, W_{f'}$) Points where the mixture components disagree. The rejecting component in isolation is vulnerable to gaming. However, as part of the mixture, since each point is already accepted with probability $\frac{1}{2}$, it is not worth attempting to game the rejecting component.
\end{itemize}
Each mechanism reduces the incidence of gaming. As such, mixtures are beneficial when the induced regions contain more negative than positive mass. This leads to the following sufficient condition:
\begin{corollary}
\label{corr:gaming_region_condition}
Let $\mathcal{S} := \{A, B, C^{\mathrm{prot}}_f, C^{\mathrm{prot}}_{f'}, W_f,
W_{f'}\}$ be the regions of Theorem~\ref{thm:lambda_def}. If $f,f'\in\mathcal{F}^{*}$, $\beta(S)\ge 0$ for every $S\in\mathcal{S}$, and $\beta(S)>0$ for at least one $S\in\mathcal{S}$, then $R_{\Delta_{Q^*}}(Q^*) < R^*_\Delta$.
\end{corollary}

\begin{proof}
By Theorem \ref{thm:lambda_def}, $\Lambda(f,f')$ is a positively weighted sum of the $\beta(S)$, so $\Lambda(f,f')>0$. The result follows from Corollary \ref{corr:mixture_vs_deterministic_optimal}.
\end{proof}

Mixing can only harm a positive Agent who is rejected by at least one of the
components on their unmodified features: every region in $\mathcal{S}$ lies
inside $\{\rx \in \gX: f(\rx)=-1 \lor f'(\rx)=-1\}$. With this insight we can derive the following bound on the cost of mixing by the clean false negative rate of each component of $Q$.
 
\begin{proposition}
\label{prop:gaming_fn}
For any $f,f'\in\mathcal{F}$, let
\begin{equation}
    \mathrm{FN}(f) := P(f(x)=-1, y=1) \nonumber
\end{equation}
be the clean false-negative mass of $f$ and 
\begin{align}
  N(f,f')& := P(A,y=-1) + 2P(B,y=-1) + P(C^{\mathrm{prot}}_f,y=-1) + \nonumber \\
  & P(C^{\mathrm{prot}}_{f'},y=-1) + P(W_f,y{=}{-}1) + P(W_{f'},y=-1). \nonumber
\end{align}
Then
\begin{equation}
    \Lambda(f,f') \;\ge\; N(f,f') - \mathrm{FN}(f) - \mathrm{FN}(f').
\end{equation}
\end{proposition}
The proof of this proposition can be found in Appendix \ref{app:proof_of_gaming_fn}.

Proposition \ref{prop:gaming_fn} highlights properties the components of a good mixture should have. The magnitude of the positive $N(f,f')$ increases when negative Agents can afford to game either component on their own, but not jointly ($A, B, C_{f}^{\text{prot}}, C_{f'}^{\text{prot}}$). This occurs when $f$ and $f'$ are gameable in distinct ways. The negative $FN$ terms indicate components should have good clean accuracy. Combining these observations: good components are accurate but diverse. This motivates the component pools used in Section \ref{sec:experiments}.

\section{Experimental Results}
\label{sec:experiments}
Theorem~\ref{thm:mixture_vs_deterministic} shows that, under certain conditions, a two component uniform mixture can achieve lower strategic risk than any deterministic classifier. It remains to be seen whether this occurs on real data. Section~\ref{sec:mixture_vs_deterministic} considers whether the improvement predicted by Theorem~\ref{thm:mixture_vs_deterministic} occurs in practice. We have argued that the effectiveness of randomisation over deterministic classifiers arises from the Agent's uncertainty over the classifier to which they must respond. Section~\ref{sec:robustness_to_uncertainty} compares this with related forms of Agent uncertainty raised in the literature, which rely on a single fixed classifier being concealed from the Agents. This explores if the benefit of randomisation is retained when Agents can aggregate information to identify the classifier. Finally, Section~\ref{sec:more_than_two_components} considers whether uniform mixtures with $k>2$ components further improve strategic accuracy.

\subsection{Experimental Setup}
\label{sec:setup}
\paragraph{Datasets and Evaluation.}
We run our experiments on 3 tabular datasets that have previously been used in the strategic classification literature: Houses~\citep{pace1997}, Give Me Some Credit (GMSC)~\citep{GiveMeSomeCredit}, and Credit Card Fraud~\citep{dal2015}. Further information about the datasets used in these experiments can be found in Appendix~\ref{app:experimental_details}. All reported strategic accuracies are computed using class reweighting to address class imbalance in the datasets.

\paragraph{Best responses.}
SERM-style relaxations \citep{levanon2021, geary2026} cannot be used for mixtures: replacing each component with a smooth surrogate creates an interaction between the components in Equation \ref{eqn:randomised_best_response} that does not exist for the real classifiers. We instead use decision trees, for which best responses can be computed directly. The positive region of a tree is a union of axis-aligned boxes (i.e., leaves), so the best response to a single tree is the cheapest positive box, and $j$, $p$ and $q$ (Equation \ref{eqn:jpq_def}) are minima over pairwise intersections of the two trees' leaves. Evaluating a pair of trees costs $O(L_fL_{f'})$ per Agent, where $L_f$ is the number of leaves of $f$, so evaluating all pairs scales quadratically with pool size.

\paragraph{Component pools and baseline.}
Following Proposition \ref{prop:gaming_fn}, we draw components from a pool of accurate
but diverse trees: the Rashomon set of near-optimal trees, computed with
PRAXIS \citep{heile2026}. Depth and Rashomon-set looseness are tuned by grid search on validation data to induce pools of between $1,000$ and $5,000$ trees. To score elements of this set we consider the fraction of Agents who achieve positive classification on tree $t_j$ while best responding to tree $t_i$; $\text{CA}(t_i, t_j)$. Each candidate configuration is scored based on its gaming-related diversity, which we define as $1-\max(M_{ij},M_{ji})$ evaluated at a radius independently calibrated per configuration to a $0.5$ gaming rate. Unless stated otherwise, uniform mixtures of two trees are used in all experiments. The deterministic baseline is chosen from the same pool as the tree with the highest strategic accuracy.

\paragraph{Classifier Selection.}
On each dataset the deterministic baseline, $\hat{f}$, is fit by solving the standard SERM objective. In the case of the mixture, the mixture's $k$ components must be chosen jointly, as the Agent's best response to the mixture depends on all the mixture components (Equation \ref{eqn:randomised_best_response}). This makes it computationally intractable to compute $\hat{Q}$ by exhaustive search. Instead we use the training data to score a sample of $2,000$ candidate subsets drawn uniformly at random without replacement. $\hat{Q}$  reported in the subsequent sections is the candidate attaining the highest score. $\hat{Q}$ is therefore a lower bound on the performance of the best mixture in the pool.

\paragraph{Cost.}
We use the scaled squared Euclidean cost,
\begin{equation}
\label{eqn:euclidean_cost_def}
    c(\rvx, \rvz) = \dfrac{2}{\epsilon^{2}}||\rvx-\rvz||^{2},
\end{equation}
computed over standardised feature values. We sweep the value for $\epsilon$ (Figure \ref{fig:accuracy_vs_radius}) and measure the validation strategic accuracy. In subsequent experiments a single $\epsilon$ is selected such that the gap between the performance of $\hat{Q}$ and $\hat{f}$ is maximised. This selects the regime in which the problems in the later sections are non-trivial.

\subsection{Mixture vs Deterministic}
\label{sec:mixture_vs_deterministic}
\begin{figure*}[t]
    \centering
    \includegraphics[width=\linewidth]{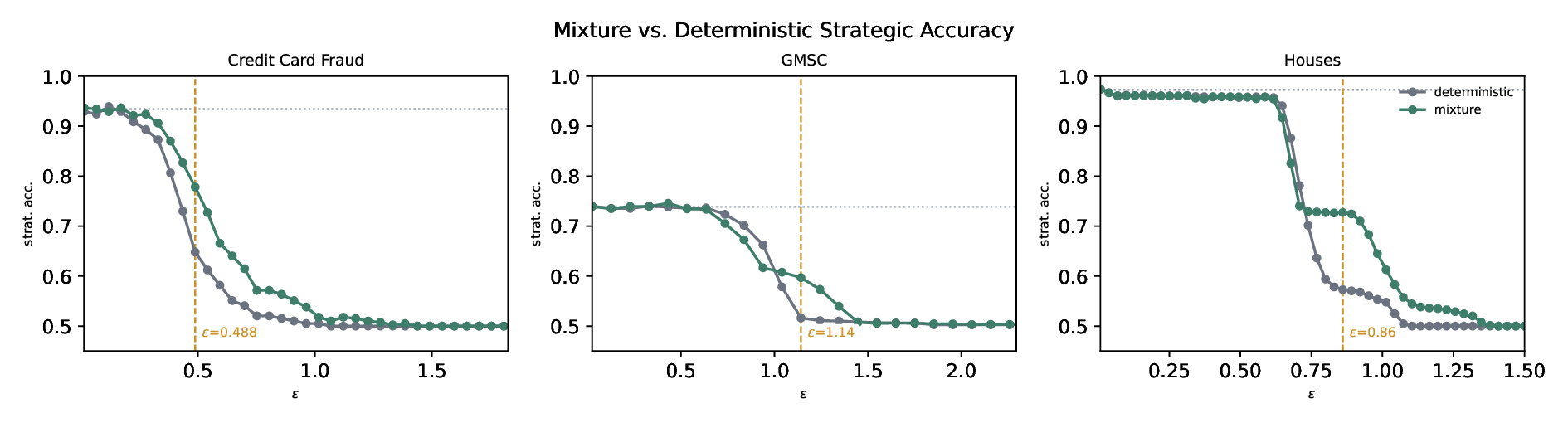}
    \caption{Comparing the strategic accuracy of the best sampled mixture against the optimal deterministic classifier for varying values of $\epsilon$ evaluated on the validation set. The value of $\epsilon$ maximising this difference is highlighted.}
    \label{fig:accuracy_vs_radius}
\end{figure*}
Figure \ref{fig:accuracy_vs_radius} compares the strategic accuracy of the optimal two component uniform mixture with the optimal deterministic classifier on three different datasets. Observe from Equation~\ref{eqn:euclidean_cost_def} that increasing $\epsilon$ makes it cheaper for Agents to perturb their state, allowing for more aggressive gaming behaviours. As $\epsilon \to 0$ gaming is prohibitively expensive, and the classifiers attain their clean accuracies. Beyond a certain value, gaming becomes so cheap that all Agents can game, and both classifiers fall to chance level ($0.5$ under class balancing).

Between these regions, there is an area where the mixture's strategic accuracy exceeds that of the deterministic classifier. These are values of $\epsilon$ where it is worthwhile for Agents to game a single deterministic classifier. However, due to the complementarity of the component pool, it is expensive to game both components of a mixture simultaneously (the Deterrence mechanism from Section~\ref{sec:mechanisms}). A positive gap is precisely the empirical counterpart to the condition specified in Theorem~\ref{thm:mixture_vs_deterministic}, with $R_{\Delta}^{*}$ estimated by $r_{\Delta_{\hat{f}}}(\hat{f})$; $\Lambda(f,f')>\gE(f) + \gE(f')$. By way of example, in the case of the Credit Card Fraud dataset, at $\epsilon=0.488$, we have $\Lambda(f,f') = 0.3318$, which is greater than $\gE(f) + \gE(f') = 0.0718$.

The $\epsilon$ values highlighted in each panel in Figure \ref{fig:accuracy_vs_radius} specify where the gap between mixture and deterministic strategic accuracy is maximised. Evaluated at a single train/validation set split, this shows that the mixture improves strategic accuracy by up to $0.13$, $0.08$, and $0.15$ on Credit Card Fraud, GMSC, and Houses respectively. The dip in the GMSC plot, where the mixture underperforms the deterministic classifier ($\epsilon \approx 0.9$), is explainable by the fact that the classifiers are selected on training data and, since mixture components are sampled without replacement, the mixture cannot match deterministic behaviour. 

These results demonstrate that the improvement predicted by Theorem~\ref{thm:mixture_vs_deterministic} does occur in practice, but only in an intermediate range of gaming costs: when gaming is either prohibitively expensive or universally affordable, the two classifiers coincide.

\subsection{Robustness to Information Aggregation}
\label{sec:robustness_to_uncertainty}
\begin{figure*}
    \centering
    \begin{subfigure}[t]{0.5\textwidth}
        \centering
        \includegraphics[width=\textwidth]{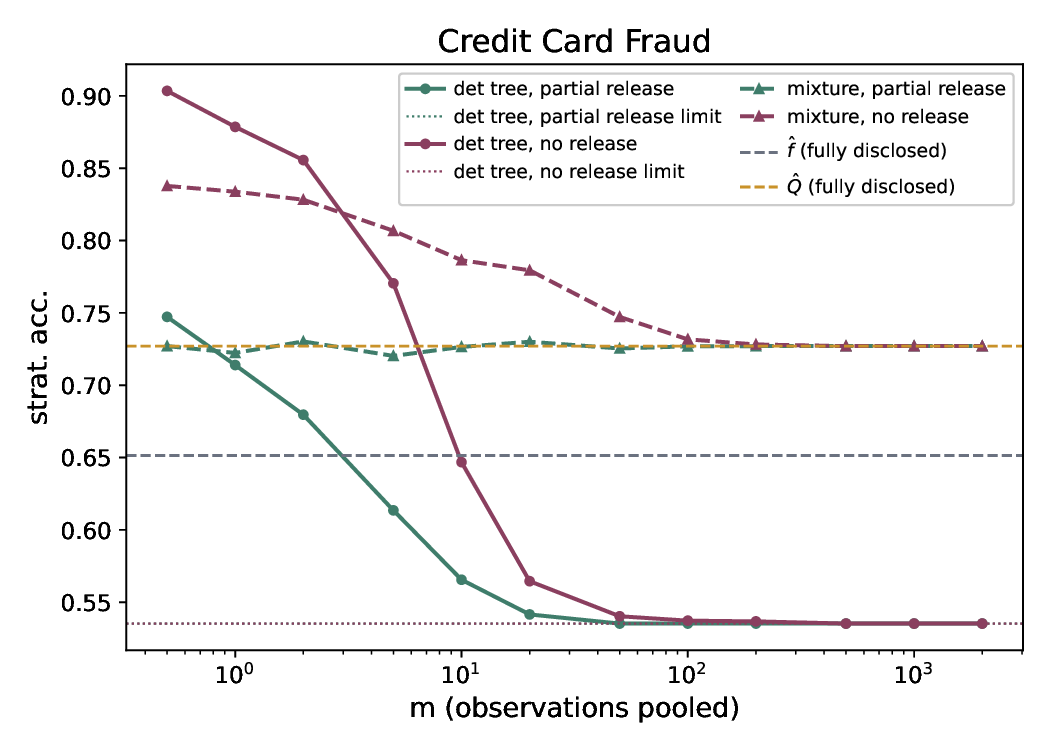}
    \end{subfigure}%
    ~ 
    \begin{subfigure}[t]{0.5\textwidth}
        \centering
        \includegraphics[width=\textwidth]{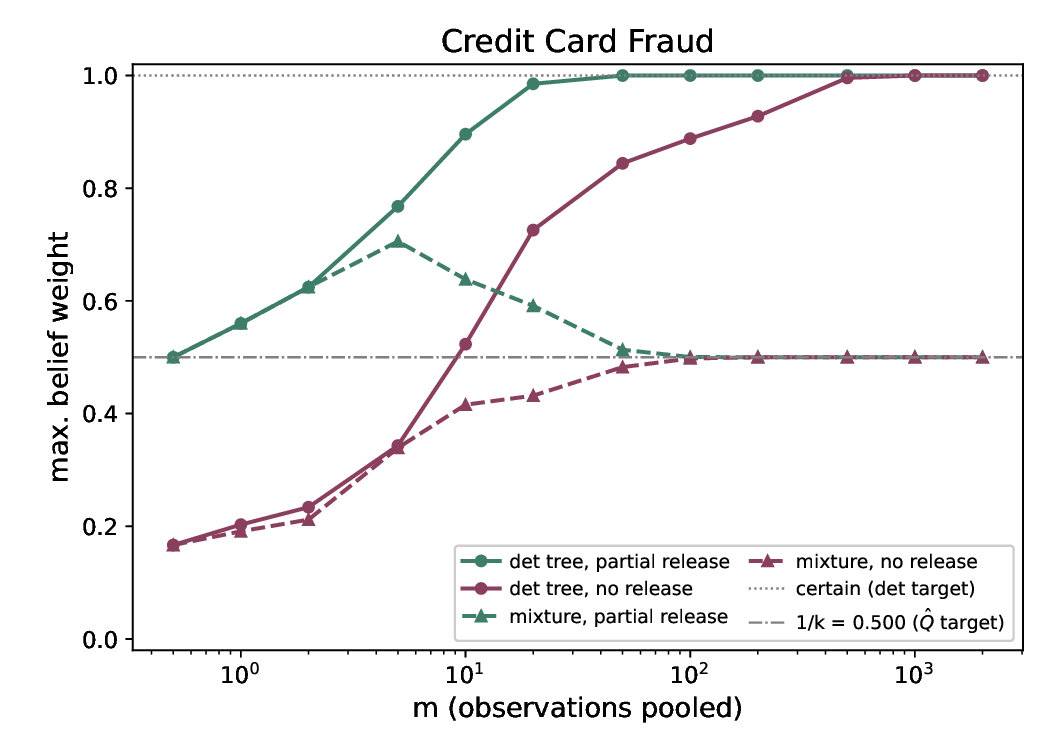}
    \end{subfigure}
    \caption{Examining the change in strategic accuracy and certainty with disclosure on the Credit Card Fraud dataset; (Left) Comparing the strategic accuracy of the classifier in the case of partial classifier release (\cite{cohen2024}), and no release (\cite{ghalme2021}) to $\hat{Q}$ (dashed yellow line) and $\hat{f}$ (dashed grey line) as the number of Agent observations ($m$) increases. Dotted lines give the strategic accuracy achieved if Agents know the classifier exactly; (Right) Tracking how Agent belief in classifier identity changes with the number of observations}
    \label{fig:robustness_to_uncertainty_credit_card_fraud}
\end{figure*}


As discussed in Section~\ref{sec:related_work}, \citet{cohen2024} and \citet{ghalme2021} exploit Agents' uncertainty about a single, fixed classifier, by disclosing a subset of candidates or nothing at all. If Agents can observe that classifier's decisions, this uncertainty can in principle be resolved, which is why transparency is usually assumed \citep{hardt2016}. Under randomisation, by contrast, it cannot.

We consider the impact of this distinction on the performance of a deterministic tree and a uniform mixture under varying levels of disclosure: full disclosure (the classifier/mixture is disclosed), partial (the Learner has a set of classifiers that could be applied, $C$, and discloses a subset $R \subseteq C$, chosen as in \citet{cohen2024}), and none ($R=C$). As in \citet{cohen2024}, Agents hold a uniform prior over the policies the Learner may deploy (any single classifier in $R$, or a uniform mixture over a size $k$-subset of $R$), and best respond to the posterior predictive. Following \citet{ghalme2021}, Agents learn from disclosed outcomes: Agents are classified sequentially and report their decisions, with each report flipped with probability $\eta =0.05$. Beliefs are updated by Bayes' rule. Each policy is selected by its strategic accuracy on the training data before any observations (where the strategic accuracy is computed with respect to the belief held by the Agents) and is never retrained. After $m$ outcomes have been pooled, Agent belief and test set strategic accuracy are reported (Figure~\ref{fig:robustness_to_uncertainty_credit_card_fraud}). Results are averaged over $20$ seeds (varying the order and draws of disclosing Agents); error bars are omitted for readability. We also report the limit of each fixed tree policy, which is the strategic accuracy realised for each classifier when it is exactly known by the Agents.

Figure~\ref{fig:robustness_to_uncertainty_credit_card_fraud}~(Left) presents the results on the Credit Card Fraud dataset. For a small number of observations, the deterministic-tree policies match or exceed $\hat{Q}$, the optimal mixture under full disclosure. On GMSC and
Houses (Appendix~\ref{app:robustness_to_uncertainty_gmsc_and_houses}) this one-shot advantage is substantially larger. However, as the Agents pool more observations, they become more certain of the deployed tree's identity (Figure~\ref{fig:robustness_to_uncertainty_credit_card_fraud} (Right)). As a consequence, Agents are better able to respond to the deterministic classifier, and strategic accuracy drops to the limit accuracy within $100$ observations, below $\hat{Q}$.

In contrast, in the randomised setting, even under total transparency, it is still impossible for Agents to determine exactly which classifier they must respond to. There is therefore not the same identifiability issue as exists in the deterministic setting, and the randomised classifier retains its robustness to strategic behaviour under disclosure: on Credit Card Fraud, the randomised accuracies converge to that of $\hat{Q}$, and the belief in any single component converges to $\frac{1}{k}$ (Figure \ref{fig:robustness_to_uncertainty_credit_card_fraud}). While concealing the mixture can provide a temporary benefit, this disappears as the Agents learn, converging to the fully disclosed value. On GMSC and Houses (Figures \ref{fig:robustness_to_uncertainty_gmsc} and \ref{fig:robustness_to_uncertainty_houses} in Appendix \ref{app:robustness_to_uncertainty_gmsc_and_houses}) beliefs plateau short of $\frac{1}{k}$. This indicates that some classifiers are indistinguishable to the Agents. In spite of this the randomised curves end within $0.05$ of $\hat{Q}$, and reliably above $\hat{f}$ as well as the limit values for the deterministic classifiers under any disclosure policy.

\cite{hikri2026} also consider classifier disclosure to an uncertain Agent in the form of the Ambiguity set. However, Agents in their setting are max-min decision-makers rather than expected-utility maximisers, so we present the comparison separately in Appendix~\ref{app:robustness_to_uncertainty_ambiguity}. Under the Ambiguity setting, the same pattern holds once Agents can distinguish between candidates, although, when gaming is cheap, this requires Agents to probe deliberately rather than to observe passively.

Hence we can conclude that the benefit of concealing a deterministic classifier remains only so long as the Agents are unable to aggregate observations and is lost once the Agents have pooled sufficient observations. In contrast, the benefit in the randomised setting does not depend on concealment.

\subsection{More Than Two Components}
\label{sec:more_than_two_components}
\begin{figure}
    \centering
        \begin{subfigure}[t]{0.33\textwidth}
        \centering
        \includegraphics[width=\textwidth]{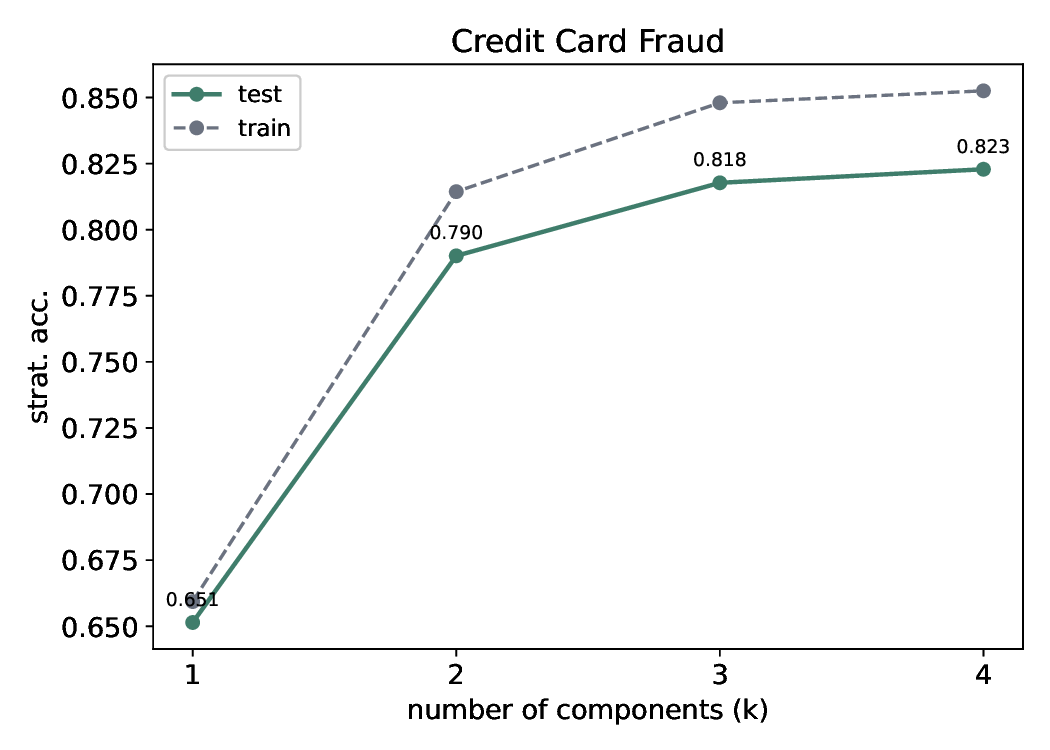}
    \end{subfigure}%
    ~
    \begin{subfigure}[t]{0.33\textwidth}
        \centering
        \includegraphics[width=\textwidth]{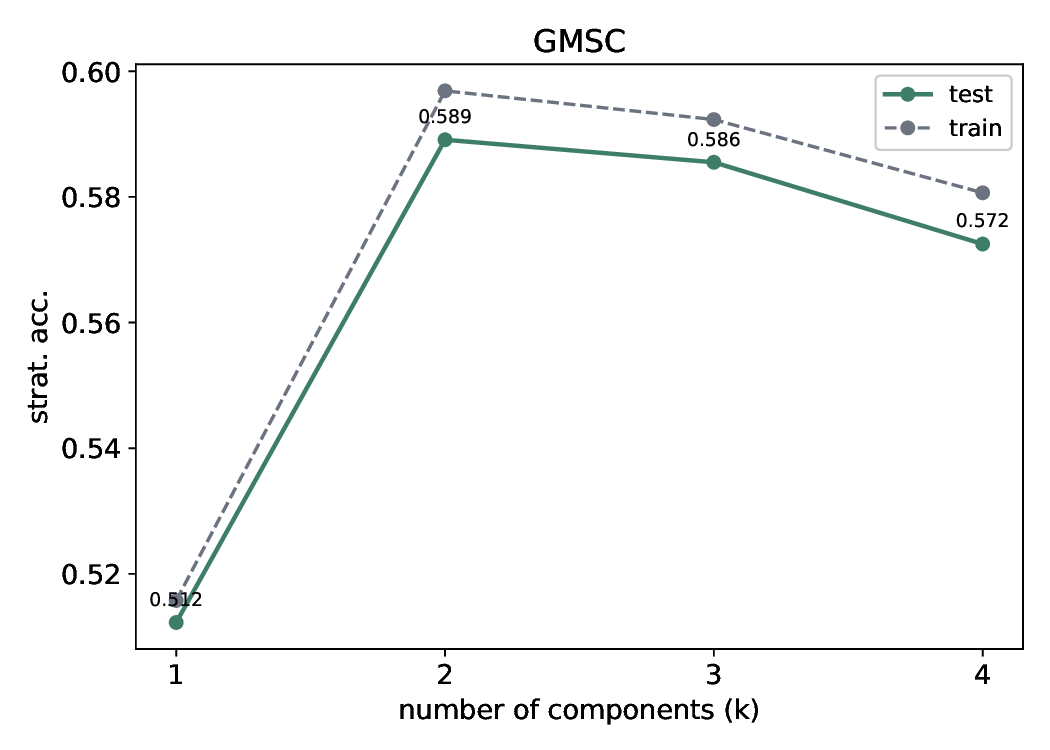}
    \end{subfigure}%
    ~ 
    \begin{subfigure}[t]{0.33\textwidth}
        \centering
        \includegraphics[width=\textwidth]{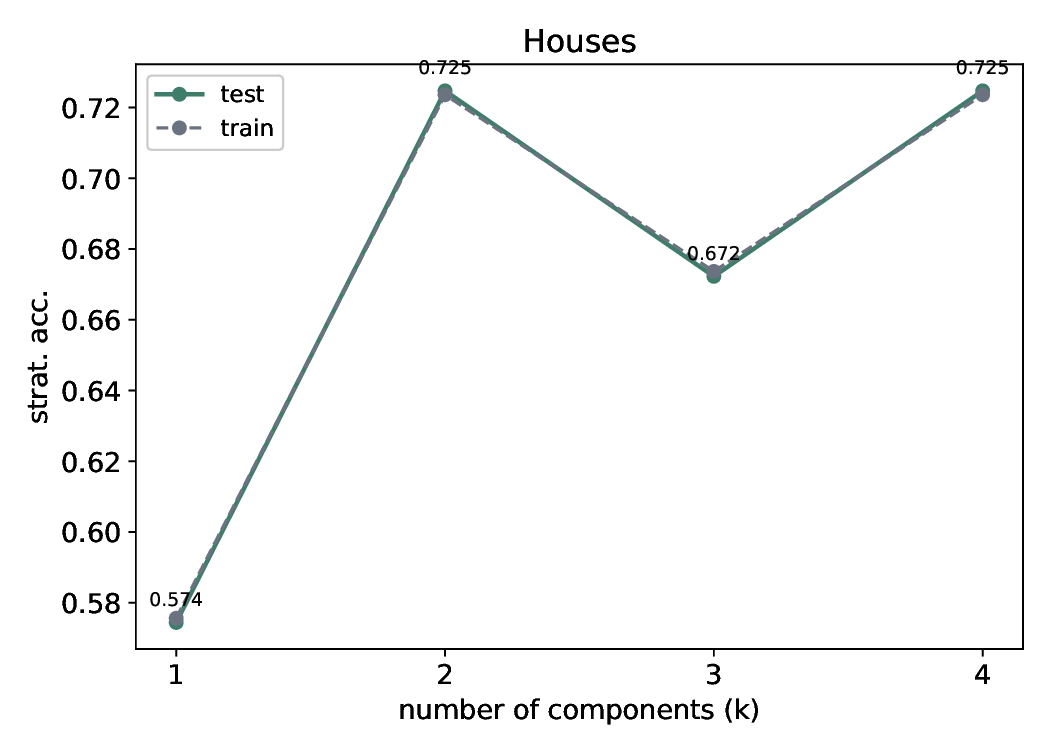}
    \end{subfigure}
    \caption{Comparing the strategic accuracy of a uniform mixture for increasing number of components, where the dashed line depicts train accuracy, and the solid line is test accuracy. Classifier selection is on train accuracy; (Left) Credit Card Fraud; (Middle) GMSC; (Right) Houses.}
    \label{fig:more_than_two_components}
\end{figure}

Most of the results presented in this work focus on mixtures with $2$ components. It remains to be considered whether mixtures with greater numbers of components provide further benefit. To examine this we select a sample pool of $N'=30$ trees from the component pools defined in Section~\ref{sec:setup}. Trees are first grouped into sets that are behaviourally identical on the training set, and the tree with the highest strategic accuracy in each set is retained. The top $N'$ of these, by strategic accuracy, form the final candidate pool. This ensures that, for any $k$, components drawn from this pool are accurate and functionally distinct, in line with Proposition \ref{prop:gaming_fn}.

We exhaustively evaluate every $k$-component uniform mixture of distinct trees from this pool, select  $\hat{Q}$ as the best on training data, and report its test strategic accuracy. Figure~\ref{fig:more_than_two_components} presents the results on each of our experiment datasets. Across each dataset there is a substantial jump going from $k=1$ to $k=2$ components, consistent with the results from Section~\ref{sec:mixture_vs_deterministic}. This supports the hypothesis that the uncertainty introduced by the randomisation hampers the Agents' ability to game the classifier. However, for $k>2$ there is limited evidence of continued benefit; Credit Card Fraud demonstrates modest strategic accuracy improvements while, for GMSC and Houses the addition of further components can even reduce strategic accuracy. The fact that train and test accuracy move together indicates this is not due to overfitting. 

Hence, for uniform mixtures, two components are sufficient to obtain the
benefit of randomisation; whether this is true for non-uniform mixtures is left open. 

\section{Discussion}
\label{sec:discussion}

\paragraph{When should we expect randomisation to help?}
Corollary \ref{corr:gaming_region_condition} provides sufficient conditions under which randomisation would be beneficial in strategic settings; namely when there are several candidate classifiers that achieve near optimal strategic risk, and negative-dominated regions in which the response to the mixture is distinct from the response to either classifier. This is further supported by Proposition \ref{prop:gaming_fn}, which relates the performance to clean classifier accuracy, and model diversity. 

This first point indicates that randomisation is likely to help in settings where the true label depends on a large number of features, and when different near-optimal classifiers rely on different subsets of them. This makes it more likely that the gaming efforts are essentially orthogonal, resulting in greater cost to game any mixture. Whereas the latter point requires that mixtures are designed such that Agents that can game one classifier, but not both, be predominantly negative. This is the same underlying principle that motivates strategic classification; if gaming were mainly done by positive Agents, the Learner would have no inclination to deter it. This is supported by our results from Figure~\ref{fig:accuracy_vs_radius} where we observed that the advantage to randomisation emerged for intermediate cost values, and from Proposition~\ref{prop:gaming_fn} relating it to the diversity of the pool.

\paragraph{Limitations.}
Our experiments rely on exact best responses, which restricts them to decision trees and makes evaluation quadratic in the size of the pool; designing efficient learners for randomised strategic classifiers is the main open problem. Furthermore, we consider only uniform mixtures, so our empirical results are a lower bound on what optimal randomised classifiers can achieve, and we only evaluate our method on three datasets. Like most work in strategic classification, we assume a known cost function shared by all Agents, and Agents who maximise expected utility; risk-averse Agents would be deterred more strongly by randomisation. Finally, randomisation is only credible if the Learner actually samples as announced, which in practice may require audited or publicly verifiable randomness.

\section*{Broader Impact Statement}
There is a natural resistance to randomisation due to the fact that, by definition, it may result in Agents with the same state being treated differently, raising valid questions about fairness.
Furthermore, when components of a mixture disagree, the result can be positive Agents who would have received a positive classification being otherwise negatively classified.
The alternative, however, is not a perfect deterministic classifier. Strategic classification already forces a choice between a classifier that is accurate on clean data but vulnerable to gaming, and a conservative classifier that deters gaming by knowingly rejecting some positive Agents, often those least able to game \citep{milli2019, hu2019}. As motivated in Theorem \ref{thm:lambda_def}, randomisation discourages gaming behaviour, which may reduce the burden on Agents who otherwise would have had no incentive to game, and make the system fairer for those who are unable to game.
Randomisation has also been argued to support fairer decisions in related settings such as lending \citep{kilbertus2020fair}. Finally, because the mixture is fully disclosed, it is compatible with transparency requirements, unlike approaches that rely on hiding the classifier.

\section{Conclusions}
We have argued that the Learner in strategic classification should be allowed to commit to a distribution over classifiers. Doing so never increases strategic risk; we characterised exactly when a uniform mixture of two classifiers reduces it, and decomposed the reduction into three mechanisms that apply to any hypothesis class and feature space. Empirically, mixtures of two decision trees outperform the best deterministic tree over an intermediate range of gaming costs and, unlike approaches that conceal a fixed classifier, retain their advantage when Agents pool their observations. Our analysis does not suggest a computationally efficient strategy for training randomised classifiers; designing such algorithms, likely tailored to specific hypothesis classes, is left to future work.



\bibliographystyle{tmlr}
\bibliography{bibliography}

\newpage

\appendix
\section{Proof of Theorem \ref{thm:lambda_def}}
\label{app:lambda_def_proof}
From Equation \ref{eqn:lambda_def} we know that $\Lambda(f,f')$ relates the agent's behaviour when responding to $f$ and $f'$ with their behaviour in response to the uniform mixture $Q$. To capture this we compare, pointwise, the classifications received when responding to the deterministic classifiers, with the classifications received when responding to the mixture, and show that the expected loss decomposes into exactly six regions.

This proof has three parts: first we express the loss difference as a simple pointwise function, $\varphi$; second, we characterise the agent's best response to $Q$ and to each classifier individually, as a function of $j,p,q$ (Equation \ref{eqn:jpq_def}); finally, we case split on the clean classification of $\rvx$ ($(f(\rvx), f'(\rvx))$) and integrate.

The proof will make use of the following result;
\begin{lemma}
    \label{lemma:varphi_values}
    For $(s,t) \in \{-1,1\}^{2}, \varphi((s,t),\ry) := l(s,\ry) + l(t,\ry)$ attains one of the following values:
    \begin{equation}
        \varphi((s,t),\ry) =
        \begin{cases}
            2 \cdot \ind[\ry=-1] \text{ if } s=t=1 \\
            2\cdot \ind[\ry=1] \text{ if } s=t=-1 \\
            1 \text{ if } s \neq t
        \end{cases}
    \end{equation}
\end{lemma}

\begin{proof}
    This is an immediate consequence of the definition of $\varphi$; if $s=t$ then both losses contribute $1$ for non-zero loss, and $0$ otherwise. In the case when $s \neq t$ then exactly one of the losses will be non-zero, resulting in a contribution of $1$ regardless of the value of $\ry$.
\end{proof}

\begin{lemma}[Individual best response]
\label{lemma:individual_best_response}
    The necessary and sufficient conditions for an agent's best response to classifiers $f, f'$ can be given as follows: 
    \begin{equation}
        f(\Delta_{f}(\rx)) = 1 \iff \min(j(\rx), p(\rx)) < 2, \qquad
        f'(\Delta_{f'}(\rx)) = 1 \iff \min(j(\rx), q(\rx)) < 2.
    \end{equation}
\end{lemma}
\begin{proof}
    An agent facing $f$ alone finds the closest $\Delta_{f}(\rx)$ such that $f(\Delta_{f}(\rx))=1$. This is achieved whether or not $f'(\Delta_{f}(\rx))=1$. This attains cost $\min(j(\rx), p(\rx))$ (the cheaper of a jointly-positive or an $f$-only-positive target). This improves on staying (utility $-1$) if and only if $1 - \min(j(\rx), p(\rx)) > -1 \iff \min(j(\rx), p(\rx)) < 2$. The argument for $f'$ is symmetric.
\end{proof}

This proof will also require us to reason about the agent's best response behaviour under the uniform mixture, $Q = U(\{f,f'\})$. For some arbitrary $\rvz \in \gX$, let
\begin{equation}
\label{eqn:v_Q_def}
    v_{Q}(\rvz) := \E_{g \sim Q}[g(\rvz)] = \dfrac{1}{2}f(\rvz) + \dfrac{1}{2}f'(\rvz).
\end{equation}
It is clear from the definition that $v_{Q}(\rvz) \in \{-1,0,1\}$. For an agent with state $\rvx \in \gX$, let $v_{0} := v_{Q}(\rvx) = \frac{1}{2}f(\rvx) + \frac{1}{2}f'(\rvx)$ be their baseline utility (the utility they achieve if they don't engage in any strategic behaviour). The agent's behaviour in response to $Q$ will be to maximise $v_{Q}(\Delta_{Q}(\rvx)) - c(\rvx, \Delta_{Q}(\rvx))$. 

With this established, the proof proceeds as follows:
\begin{proof}
    Let $(a,b)(\rvx) = (f(\Delta_{f}(\rvx)), f'(\Delta_{f'}(\rvx)))$ the classifications received by $\rvx \in \gX$ under $f$ and $f'$, respectively, when the agents respond to each classifier individually. Let $(c,d)(\rvx) = (f(\Delta_{Q}(\rvx)), f'(\Delta_{Q}(\rvx)))$ be the classifications received by $\rvx$ under $f$ and $f'$ when the agents behaviour, $\Delta_{Q}$, is in response to the uniform mixture $Q=U(\{f,f'\})$. Let $\varphi((s,t),y) = l(s,y) + l(t,y)$, then $R_{\Delta_{f}}(f) + R_{\Delta_{f'}}(f') = \E[\varphi((a,b),y)]$, $2R_{\Delta_{Q}}(Q) = \E[\varphi((c,d),y)]$ and Equation \ref{eqn:lambda_def} can be rewritten as 
    \begin{equation}
        \Lambda(f,f') = R_{\Delta_{f}}(f) + R_{\Delta_{f'}}(f') - 2R_{\Delta_{Q}}(Q) = \E[\varphi((a,b),\ry) - \varphi((c,d),\ry)] = \E[\psi(\rvx)].
    \end{equation}
    where $\psi(\rvx, \ry) := \varphi((a,b),\ry) - \varphi((c,d),\ry)$. The goal is to derive a closed-form definition for $\E[\psi(\rvx)] = \int_{\gX}\psi(\rvx)$.
    
    To achieve this exhaustively enumerate all the possible values for $((a,b), (c,d))$ that can be realised in response to rational strategic behaviours and consider the resulting value of $\psi$. For clarity we will group these by the possible initial behaviours of $(f(\rvx), f'(\rvx))$;
    \begin{itemize}
        \item $(f(\rvx), f'(\rvx)) = (1,1)$: In this case both classifiers are already satisfied ($v_{0}=1$), so no gaming behaviour is required; $(a,b)=(c,d)=(1,1) \implies \psi(\rvx,\ry) = 0$.
        \item $(f(\rvx),f'(\rvx)) = (1,-1)$: Begin by reasoning about the behaviour of $\Delta_{f'}$. Per Lemma \ref{lemma:individual_best_response}, $b=1 \iff j<2 \text{ or } q<2 \iff \min(j,q)<2$.
        Reasoning about the behaviour responding to the mixture, $\Delta_{Q}$, observe that $v_{0}=\frac{1}{2}f(\rvx) + \frac{1}{2}f'(\rvx)=0$. $v_{Q}(\Delta_{Q}(\rvx))>v_{0}$ therefore only holds if $v_{Q}(\Delta_{Q}(\rvx)) = 1 \iff d=f'(\Delta_{Q}(\rvx))=1$. By definition occurs only if $v_{Q}(\Delta_{Q}(\rvx))-c(\rvx, \Delta_{Q}(\rvx))>v_{0} \implies 1-j>0 \implies j<1$.
        Putting these conditions together ($\min(j,q)<2$ and $j<1$), and using the result from Lemma \ref{lemma:varphi_values} we can enumerate the possible outcomes and the resulting possible values for $\psi$:
        \begin{itemize}
            \item $j<1$: $\implies \min(j,q)\leq j<2 \implies b=1, d=1 \implies \psi(\rvx, \ry) = 0$
            \item $j\geq 1, \min(j,q)\geq 2$: $\implies b=-1, d=-1 \implies \psi(\rvx, \ry) = 0$
            \item $j \geq 1, \min(j, q)<2$: $\implies b=1 \implies \varphi((a,b),\ry) = 2 \cdot \ind[\ry=-1]$. $d=-1 \implies \varphi((c,d),\ry) = 1$.
        \end{itemize}
        When evaluating $\int_{\gX}[\psi(\rvx)]$ in this region, the resulting contribution is
        \begin{align}
        \label{eqn:lambda_def_proof_beta_W_f}
            &2P(j\geq1, \min(j,q)<2, y=-1) - P(j\geq1, \min(j,q)<2) \nonumber \\
            &= P(j\geq1, \min(j,q)<2, y=-1) - P(j\geq1, \min(j,q)<2, y=1) \nonumber \\
            &= \beta(W_{f}).
        \end{align}
        \item $(f(\rvx), f'(\rvx)) = (-1,1)$: By symmetry with the previous argumentation, this contributes $\beta(W_{f'})$ to $\int_{\gX}[\psi(\rvx)]$.
        \item$(f(\rvx), f'(\rvx)) = (-1,-1)$: Begin by reasoning about the behaviours of $\Delta_{f}$ and $\Delta_{f'}$. Per Lemma \ref{lemma:individual_best_response}:
        \begin{align}
        \label{eqn:lambda_def_proof_det_conditions}
            a=1 &\iff \min(j,p)<2, \nonumber \\
            b=1 &\iff \min(j,q)<2.
        \end{align}
        To reason about the behaviour of $\Delta_{Q}$ observe that $v_{0}=-1$ therefore, $v_{Q}(\Delta_{Q}(\rvx)) = 0 \text{ }(f(\Delta_{Q}(\rvx)) = -f'(\Delta_{Q}(\rvx)))$ or $v_{Q}(\Delta_{Q}(\rvx))=1 \text{ } (f(\Delta_{Q}(\rvx))=f'(\Delta_{Q}(\rvx))=1)$ are both possible rational behaviours that must be considered. Let $u_{I} := \sup_{\rvz \in \gX: v_{Q}(\rvz)=0}[0-c(\rvx,\rvz)] = -\min(p,q)$ be the utility of a response that games only one of $f$ and $f'$, and let $u_{J} := 1-j$ be the utility of a response that games both classifiers. The $v_{Q}(\Delta_{Q}(\rvx))=0$ response is optimal if and only if $u_{I}>v_{0}$ and $u_{I}>u_{J}$. Assuming that ties break in favour of inaction, the agent's best response is to do nothing ($v_{Q}(\Delta_{Q}(\rvx))=-1$) if and only if $v_{0}\geq u_{I}$ and $v_{0} \geq u_{J}$. This gives the conditions:
        \begin{align}
        \label{eqn:lambda_def_proof_mix_conditions}
            v_{Q}(\Delta_{Q}(\rvx)) = 1 &\iff \min(p,q)>j-1 \land j<2\\
            v_{Q}(\Delta_{Q}(\rvx)) = 0 &\iff \min(p,q)<j-1 \land \min(p,q)<1 \\
            v_{Q}(\Delta_{Q}(\rvx)) = -1 &\iff \min(p,q)\geq 1 \land j\geq2.
        \end{align}
        
        Next enumerate the possible combinations for these cases:
        \begin{itemize}
            \item $v_{Q}(\Delta_{Q}(\rvx))=1$ ($\min(p,q)>j-1 \land j<2$): $\implies \min(j,p)\leq j <2, \min(j,q) \leq j <2 \implies (a,b)=(c,d)=(1,1) \implies \psi(\rvx)=0$.
            \item $v_{Q}(\Delta_{Q}(\rvx))=0$ ($\min(p,q)<j-1 \land \min(p,q)<1$: WLOG assume $p<q$. Then $(c,d) = (1,-1)$. $p<1 \implies \min(j,p)\leq p < 2 \implies a=1$. Therefore it remains to examine the conditions determining the value of $b$:
            \begin{itemize}
                \item $\min(j,q)<2$: $(a,b)=(1,1)$. By Lemma \ref{lemma:varphi_values} $\varphi((a,b),\ry) = 2 \cdot \ind[\ry=-1]$. $(c,d)=(1,-1) \implies \varphi((c,d),\ry) = 1$. Integrating over this region gives:
                \begin{align}
                    &2P(\min(p,q)<j-1, \min(p,q)<1, \min(j,q)<2, y=-1) \nonumber\\
                    & \qquad \qquad- P(\min(p,q)<j-1, \min(p,q)<1, \min(j,q)<2) \nonumber \\
                    & = P(\min(p,q)<j-1, \min(p,q)<1, \min(j,q)<2, y=-1) \nonumber \\
                    & \qquad \qquad - P(\min(p,q)<j-1, \min(p,q)<1, \min(j,q)<2, y=1) \nonumber \\
                    & = P(\min(p,q)<j-1, \min(p,q)<1, (j<2 \text{ or } \max(p,q)<2), y=-1) \nonumber\\
                    & \qquad \qquad- P(\min(p,q)<j-1, \min(p,q)<1, (j<2 \text{ or } \max(p,q)<2), y=1) \nonumber \\
                    & = \beta(A).
                \end{align}
                So this region contributes $\beta(A)$ to the integral.
                \item $min(j,q) \geq 2$: $(a,b)=(1,-1)$. By Lemma \ref{lemma:varphi_values} $\varphi((a,b), \ry) = \varphi((c,d),\ry) = 1 \implies \psi(\rvx) = 0$.
            \end{itemize}
            \item $v_{Q}(\Delta_{Q}(\rvx))=-1$ ($\min(p,q)\geq 1 \land j\geq2$): By definition $(c,d)=(-1,-1)$. From Equation \ref{eqn:lambda_def_proof_det_conditions}; $j\geq 2 \implies a=1 \iff p<2, b=1 \iff q<2$. Examine the possible behaviours:
            \begin{itemize}
                \item $\min(p,q)\geq 1, p<2, q<2$: $(a,b)=(1,1) \implies \varphi((a,b),\ry) = 2 \ind[y=-1]$. $(c,d)=(-1,-1) \implies \varphi((c,d),\ry) = 2 \cdot \ind[y=1]$. Integrating over this region gives:
                \begin{align}
                    &2P(j \geq2, \min(p,q)\geq 1, p<2, q<2, y=-1) - 2P(j \geq2, \min(p,q)\geq 1, p<2, q<2, y=1) \nonumber \\
                    &= 2[P(j \geq 2, p \in [1,2), q \in [1,2), y=-1) - P(j \geq 2, p \in [1,2), q \in [1,2), y=1)] \nonumber \\
                    &= 2\beta(B).
                \end{align}
                \item $\min(p,q)\geq 1, p<2, q\geq2$: $(a,b)=(1,-1), \varphi((a,b),\ry) = 1$, $(c,d)=(-1,-1), \varphi((c,d),\ry)=2 \cdot \ind[\ry=1]$. By the same process as involved in deriving $\beta(W_{f})$ (Equation \ref{eqn:lambda_def_proof_beta_W_f}) this contributes $\beta(C_{f}^{\text{prot}})$ to the integral. 
                \item $\min(p,q)\geq 1, p\geq2, q<2$: Symmetrically this contributes $\beta(C_{f'}^{\text{prot}})$
                \item $\min(p,q)\geq 1, p\geq2, q\geq2$: $(a,b)=(c,d)=(-1,-1) \implies \psi(\rvx)=0$ in this region. 
            \end{itemize}
        \end{itemize}
    \end{itemize}

    Finally, note that each of the regions identified in the exhaustive search are mutually exclusive, but comprehensively cover all possible rational gaming behaviours. As such, $\E[\psi(\rvx)] = \int_{\gX} [\psi(\rvx)]$ can be computed exactly by summing up the contributions from the distinct regions identified above. This gives:
    \begin{equation}
        \Lambda(f,f') = \beta(A) + 2\beta(B) + \beta(C_{f}^{\text{prot}}) + \beta(C_{f'}^{\text{prot}}) + \beta(W_{f}) + \beta(W_{f'})
    \end{equation}
    as required.
\end{proof}

\section{Interpretation of Gaming Behaviour Regions}
\label{app:gaming_region_interpretation}
 
Table \ref{tab:regions} summarises how Agents in each region of
Theorem \ref{thm:lambda_def} behave when facing $f$ alone, $f'$ alone, and the
mixture $Q = U(\{f,f'\})$, together with the change this causes in the number of
errors summed over the two components (equivalently, twice the expected error of
the mixture). In every region, mixing reduces errors on negative Agents and
increases them on positive Agents by the same amount, which is why each region
contributes a multiple of $\beta$ to $\Lambda$.
 
\begin{table}[h]
\centering
\small
\begin{tabular}{lccccc}
\hline
Region & $(f(x),f'(x))$ & vs.\ $f$ & vs.\ $f'$ & vs.\ $Q$ & Errors ($y{=}{-}1$ / $y{=}1$) \\
\hline
$A$ (Dilution) & $(-1,-1)$ & games & games & games one & $2\to1$ / $0\to1$ \\
$B$ (Deterrence) & $(-1,-1)$ & games & games & stays & $2\to0$ / $0\to2$ \\
$C^{\text{prot}}_f$ (Deterrence) & $(-1,-1)$ & games & stays & stays & $1\to0$ / $1\to2$ \\
$W_f$ (Settling) & $(1,-1)$ & stays & games & stays & $2\to1$ / $0\to1$ \\
\hline
\end{tabular}
\caption{Behaviour in each region of Theorem \ref{thm:lambda_def}
($C^{\text{prot}}_{f'}$ and $W_{f'}$ are symmetric). ``Errors'' gives the
number of errors summed over the two components when Agents respond to the
components individually $\to$ when they respond to the mixture.}
\label{tab:regions}
\end{table}

\subsection{Interpreting Gaming Behaviour}
Theorem \ref{thm:lambda_def} provides a per-region decomposition of $\Lambda(f,f')$. This relates the behaviour of $\Delta_{Q}$, the response to the mixture, to $\Delta_{f}$ and $\Delta_{f'}$, the response to the component classifiers. To facilitate interpretation of this definition it is necessary to reason about the deterministic gaming behaviour $\Delta_{f}$. To begin we define the regions where $f \text{ and } f'$ disagree:
\begin{align}
    \label{eqn:d_definition}
    D_{+} := \{\rvx: f(\rvx) = 1, f'(\rvx) = -1\} \nonumber \\
    D_{-} := \{\rvx: f(\rvx) = -1, f'(\rvx) = 1\}.
\end{align}

Next consider the gaming behaviour. Consider the following behaviour for the minimum cost to game a classifier.
\begin{definition}
\label{def:m_f}
For a classifier $g$, let $m_g(\rvx) := \inf\{c(\rvx,\rvz) : g(\rvz)=1\}$ denote the cost for $\rvx$ to reach the nearest point at which $g$ outputs $1$, considering $g$ in isolation.
\end{definition}

\begin{corollary}
\label{corr:m_f_as_min_jp}
    \begin{align}
        m_{f} &= \min(j,p) \nonumber \\
        m_{f'} &= \min(j,q).
    \end{align}
\end{corollary}

\begin{proof}
    By definition
    \begin{align}
        m_f(\rvx) &= \inf\{c(\rvx,\rvz) : f(\rvz)=1\} \nonumber \\
        &= \min(\inf\{c(\rvx,\rvz) : f(\rvz)=1, f'(\rvz)=1\}, \nonumber\\
        &\quad \quad \inf\{c(\rvx,\rvz) : f(\rvz)=1, f'(\rvz)=-1\}) \nonumber \\
        &= \min(j(\rvx),p(\rvx)).
    \end{align}
    This reasoning holds symmetrically for $m_{f'}$.
\end{proof}

It follows that we can define the set of points that $f$ is vulnerable to gaming from as
\begin{align}
\label{eqn:g_f_def}
    G_{f} &:= \{\rvx: \exists \rvz, c(\rvx, \rvz)<2, f(\rvx)=-1, f(\rvz)=1\} \nonumber\\
    &= \{\rvx: f(\rvx)=-1, m_{f}(\rvx)<2\}
\end{align}
be the set of points in $\gX$ that can game $f$, and 
\begin{equation}
    \label{eqn:c_f_def}
    C_{f} := \{\rvx: f(\rvx)=-1, m_{f}(\rvx)<1\}
\end{equation}
the set of points in $\gX$ that can \textit{cheaply} game $f$ (that is $c(\rvx, \Delta_{f}(\rvx))<1$). It follows that $E_{f} := G_{f} \setminus C_{f}$ is the set of points that \textit{expensively} game $f$. 

It is possible to reason about the mixture gaming behaviour, $\Delta_{Q}$, in a similar manner;
\begin{align}
\label{eqn:g_f_f_prime_def}
    G_{f,f'} &:= \{\rvx: f(\rvx)=f'(\rvx)=-1, v_{Q}(\Delta_{Q}(\rvx))=1\} \nonumber \\
    &= \{\rvx: f(\rvx)=f'(\rvx)=-1, \min(p,q)>j-1, j<2\}.
\end{align}
$G_{f,f'}$ is the joint analogue of $G_{f}$, characterising the set of points that can feasibly game both components of the uniform mixture, $Q$. The second line follows from the condition derived in Equation \ref{eqn:lambda_def_proof_mix_conditions}. As in the deterministic case, we can also consider the set of points that can \textit{cheaply} jointly game $f$ and $f'$
\begin{equation}
    \label{eqn:c_f_f_prime_def}
    C_{f,f'} := \{\rvx: j<1\}.
\end{equation}

With these definitions established we consider the following 
\begin{lemma}
\label{lemma:c_prot_def}
    $C_f^{\text{prot}} = (E_f \cap \{\rvx: f'(\rvx)=-1\}) \setminus G_{f'}$
\end{lemma}
\begin{proof}
    From Equations \ref{eqn:c_f_def} and \ref{eqn:g_f_def},
    \begin{align}
        E_{f} \cap \{\rvx: f'(\rvx)=-1\} &= (G_{f} \setminus C_{f})\cap\{\rvx: f'(\rvx)=-1\} \nonumber \\
        &= \{\rvx: f(\rvx)=f'(\rvx)=-1,\ m_{f}\in[1,2)\}.
    \end{align}
    Subtracting $G_{f'} = \{\rvx: f'(\rvx)=-1,\ m_{f'}<2\}$ then gives
    \begin{equation}
        (E_f \cap \{\rvx: f'(\rvx)=-1\}) \setminus G_{f'} = \{\rvx: f(\rvx)=f'(\rvx)=-1,\ m_{f}\in[1,2),\ m_{f'}\geq2\}.
    \end{equation}
    Recall from Corollary \ref{corr:m_f_as_min_jp} that $m_f = \min(j,p)$ and $m_{f'}=\min(j,q)$. Substituting these in,
    \begin{align}
        &\{\rvx: f(\rvx)=f'(\rvx)=-1,\ \min(j,p) \in [1,2),\ \min(j,q) \geq 2\} \nonumber\\
        &\quad= \{\rvx: f(\rvx)=f'(\rvx)=-1,\ \min(j,p) \in [1,2),\ j \geq 2,\ q\geq 2\} \nonumber\\
        &\quad= \{\rvx: f(\rvx)=f'(\rvx)=-1,\ p \in [1,2),\ j \geq 2,\ q\geq 2\} = C_{f}^{\text{prot}},
    \end{align}
    where the second line uses $\min(j,q)\geq2 \iff j\geq2 \text{ and } q\geq2$, and the third uses that $j\geq2$ together with $\min(j,p)<2$ forces $p<j$, so $\min(j,p)=p$.
\end{proof}
With the above holding symmetrically for $C_{f'}^{\text{prot}} = (E_{f'} \cap \{\rvx: f(\rvx)=-1\}) \setminus G_{f}$. 

\begin{lemma}
    \label{lemma:w_f_def}
     $W_f = (D_{+}\cap G_{f'}) \setminus C_{f,f'}$ (symmetrically $W_{f'}$).
\end{lemma}
\begin{proof}
    From the definitions of $D_{+}$ (Equation \ref{eqn:d_definition}), $G_{f'}$ (Equation \ref{eqn:g_f_def}), and $C_{f,f'}$ (Equation \ref{eqn:c_f_f_prime_def}
    \begin{align}
        &(D_{+} \cap G_{f'}) \setminus C_{f,f'} = (D_{+} \cap G_{f'})\cap \{\rvx: j\geq1\} \nonumber \\
        & \qquad = \{f(x)=1,\ f'(x)=-1,\ \min(j,q)<2,\ j\geq1\} \nonumber \\
        & \qquad = W_{f}
    \end{align}
\end{proof}

\section{Proof of Proposition \ref{prop:gaming_fn}}
\label{app:proof_of_gaming_fn}
 
\begin{proof}
By Theorem~\ref{thm:lambda_def} and the definition of $\beta$,
$\Lambda(f,f') = N(f,f') - P_1$, where
\begin{equation}
    P_1 := P(A,y{=}1) + 2P(B,y{=}1) + P(C^{\text{prot}}_f,y{=}1)
+ P(C^{\text{prot}}_{f'},y{=}1) + P(W_f,y{=}1) + P(W_{f'},y{=}1).
\end{equation}
Let $K := P(A,y{=}1) + P(B,y{=}1) + P(C^{\text{prot}}_f,y{=}1) +
P(C^{\text{prot}}_{f'},y{=}1)$, so that
$P_1 = K + P(B,y=1) + P(W_f,y=1) + P(W_{f'},y=1)
\leq 2K + P(W_f,y=1) + P(W_{f'},y=1)$.
$A$, $B$, $C^{\text{prot}}_f$, $C^{\text{prot}}_{f'}$ and $W_{f'}$
are pairwise disjoint and contained in $\{x: f(x)=-1\}$, so
$K + P(W_{f'},y=1) \leq P(f(x)=-1, y=1) = \text{FN}(f)$. Symmetrically, since
$W_f \subseteq \{x : f'(x) = -1\}$,
$K + P(W_f,y{=}1) \leq \text{FN}(f')$. Adding the two inequalities gives
$P_{1} \leq \text{FN}(f) + \text{FN}(f')$, and hence
$\Lambda(f,f') \geq N(f,f') - \text{FN}(f) - \text{FN}(f')$.
\end{proof}

\section{Experimental Details}
\label{app:experimental_details}
\keypoint{Datasets} The experiments in this paper were all run on tabular datasets that have previously appeared in the strategic classification literature. 
\begin{itemize}
    \item Credit Card Fraud~\citep{dal2015}: this dataset contains credit card transactions some of which have been flagged as fraudulent. We undersample this dataset to address the class imbalance, leaving $984$ rows, containing $29$ features. 
    \item Give Me Some Credit~\citep{GiveMeSomeCredit}: contains information used to derive credit scores. The objective to identify which rows correspond to individuals in financial distress. We use a class balanced subsample of the dataset containing $16,714$ rows, each with $10$ features. 
    \item Houses~\citep{pace1997}: we use the binarized variant of the Houses dataset as used in \citet{cyffers2024}. The task specified by the dataset is to determine whether a house price is high. This dataset contains $20,640$ rows and $8$ numerical features per row. As the dataset was already relatively class balanced ($43.18\%$ positive class) no class balancing was done on this dataset.
\end{itemize}
All dataset features are standardised before the experiments are run, and a $60:20:20$ train:val:test split is used for all datasets. Standardisation is performed on the entire dataset before splits are constructed. As stated in the main paper, the accuracy reported on every experiment is the class-reweighted accuracy. 

\begin{table}[h]
    \centering
    \begin{tabular}{|c|c|c|c|c|}
        \hline
        Dataset & depth & $\lambda$\_reg & rashomon\_mult & pool size \\
        \hline
         Credit Card Fraud &  4 & 0.01 & 0.1 & 2093\\
         GMSC & 2 & 0.01 & 0.1 & 4344\\
         Houses & 3 & 0.01 & 0.15 & 212\\
         \hline
    \end{tabular}
    \caption{Hyperparameters used to construct the component pools per dataset}
    \label{table:component_pool_hyperparameters}
\end{table}
\keypoint{Component Pools} The pools from which classifiers and components were selected were Rashomon sets fit to each dataset. The PRAXIS Python library was used to construct each set \citep{heile2026}. We treated three library arguments as hyperparameters that were optimised by hyperparamter search as part of our experiments:
\begin{itemize}
    \item \textbf{depth budget}: The maximum depth that trees in the set could attain. We considered values from the set $\{2,3,4,5,6\}$.
    \item $\lambda$ (\texttt{lambda\_reg}): A regularisation constant used in the optimisation objective used in the PRAXIS tool. Controls the effective complexity of the resulting trees. We considered $\lambda \in \{0.01, 0.001\}$.
    \item \textbf{relaxation multiplier} (\texttt{rashomon\_mult}): This relaxes how close to ``optimal'' a tree needs to be to be included in the set. Larger coefficients allow for less accurate trees. The values considered were $\{0.05, 0.10, 0.15\}$.
\end{itemize}
Hyperparamters were chosen by exhaustive search through all possible combinations for a single seed and with \texttt{n\_estimators}$=200$. To ensure that a sufficiently diverse pool was realised whilst still being computationally tractable, combinations that resulted in pool sizes outside the range $[1,000,5,000]$ were discarded. Remaining combinations were scored on their diversity, as defined by the complementarity of their gaming behaviour. To score the complementarity of the gaming behaviour for a pair of trees $t_{1},t_{2}$ we consider as the proportion of points that, after gaming tree $t_{1}$, would achieve a positive classification under tree $t_{2}$, $\text{CA}(t_{1},t_{2})$;
\begin{equation}
    \text{CA}(t_{1},t_{2}) := \frac{1}{n}\sum_{i=1}^{n} \ind [t_{2}(\Delta_{t_{1}}(x_{i})) = 1].
\end{equation}
Complementarity can then be defined as $1-\max_{i,j}(\text{CA}(i,j), \text{CA}(j,i))$ where the maximum is computed across pairs of trees in the set. The hyperparameter combination that attains the maximum complementarity is selected. This is treated as a heuristic for diversity, motivating the selection of hyperparamters that induce component pools containing trees that produce distinct gaming behaviours. All values are computed on the validation set. Table~\ref{table:component_pool_hyperparameters} provides the chosen hyperparameter configurations resulting from this optimisation.

In order to compute the gaming behaviour above, it is necessary to have a cost to associate with the gaming behaviour. We tune this by specifying the value of $\epsilon$ in Equation~\ref{eqn:euclidean_cost_def}. We achieve this by fitting a tree on the training data without any gaming behaviour, and selecting an $\epsilon$ that gets closest to a $50\%$ gaming rate on this tree evaluated on the validation set. Note that this value was used exclusively to set these hyperparameters, and is distinct from the value of $\epsilon$ that was used in the remainder of the experiments.

\begin{table}[t]
    \centering
    \begin{tabular}{|c|c|c|c|c|c|c|}
        \hline
        Dataset & min $\epsilon$ & max $\epsilon$ & step-size & $\epsilon^{*}$ & mix gain\\
        \hline
         Credit Card Fraud & 0.013  & 2.6 & 0.053 & 0.488 & 0.13\\
         GMSC & 0.025 & 5.0 & 0.102 & 1.1418 & 0.081\\
         Houses & 0.008 & 1.5 & 0.031 & 0.861 & 0.154\\
         \hline
    \end{tabular}
    \caption{$\epsilon$-sweep ranges and resulting optimal $\epsilon^*$ used to run experiment detailed in Section \ref{sec:mixture_vs_deterministic}}
    \label{table:epsilon_hyperparameters}
\end{table}

\keypoint{$\epsilon$-Sweep} Figure~\ref{fig:accuracy_vs_radius} compares the strategic accuracy of a deterministic classifier against that of a randomised classifier for a range of values of $\epsilon$, which parametrises the cost associated with strategic behaviour (Equation~\ref{eqn:euclidean_cost_def}). The range of values of $\epsilon$ considered is determined by identifying an anchor value, $\epsilon_{\text{anc}}$ according to the $50\%$ gaming rate procedure used to set the hyperparamter for the Component Pool hyperparameter setting. $50$ points are then evenly selected from the range $[\frac{\epsilon_{\text{anc}}}{20}, 10\cdot\epsilon_{\text{anc}}]$.

The fixed value of $\epsilon$ used in the experiments in Section~\ref{sec:mixture_vs_deterministic}, Section~\ref{sec:robustness_to_uncertainty}, and Appendix~\ref{app:robustness_to_uncertainty_ambiguity}, $\epsilon^{*}$, is selected as the value from the swept range that maximises the mixture gain (the difference in strategic accuracy between the optimal mixture and the optimal deterministic classifier), as evaluated on the validation set. Table~\ref{table:epsilon_hyperparameters} gives the ranges and step-size considered in this parameter search, as well as the values of $\epsilon^{*}$ and the mixture gain at those values. 

\keypoint{Disclosure Experiments (Section~\ref{sec:robustness_to_uncertainty})} The experiments in Section~\ref{sec:robustness_to_uncertainty} investigated changes in Agent behaviour relating to differing levels of awareness of the identity of the ``true'' classifier. The central component of this experiment was $C$, the set of possible candidate classifiers. $C$ was built once per dataset on the training set. The steps to construct this set are as follows:
\begin{enumerate}
    \item Collapse trees with identical train-prediction vectors and select a single representative from each set (chosen according to highest train strategic accuracy).
    \item Initialise $C$ to contain the support of $\hat{Q}$. This ensures that the ``true'' components, in the case of the mixture are included in the set.
    \item Greedily add representatives to $C$ that have the largest minimum train-disagreement rate against every other member of $C$. This heuristically encourages behavioural diversity in the candidate set, making it easier for the Agents to distinguish between them.
    \item Stop adding components once $|C|=K=6$.
\end{enumerate}

$R \subseteq C$ is the set of candidate classifiers that are disclosed to the Agent. In the case of no disclosure $R=C$. In the case of partial disclosure $R \subset C$ is chosen by exhaustively evaluating all possible combinations for $|R|=k$ and selecting the combination that achieves the best train strategic accuracy assuming a uniform belief over $R$. This computation is feasible for $|C|\leq 6$. In the case of full disclosure $R=\{h\}$ where $h\in C$.

Given the disclosure type, the hypothesis space is every non-empty subset of $R$ of size $\leq k$ (where $k= |\text{supp}(\hat{Q})|$), with a uniform prior over all of them. The likelihood of an observed decision given a hypothesis $S$, $f_{S}(z)$ is the fraction of the trees in $S$ that accept $z$. The noisy observation model is defined as $P(\text{obs}=1|z,S) = (1-\eta)f_{S(z)} + \eta(1-f_{S(z)})$ for a fixed observation-noise rate $\eta=0.05$. The posterior is this likelihood times the uniform prior. If $\pi(S)$ is the unnormalised posterior probability of hypothesis $S$, then the effective per tree weight, $w_{h}$, is given as: $w_{h} = \sum_{\{S| h \in S\}} \frac{\pi(S)}{|S|}$.

We consider $L=2000$ observation steps, where agents are sampled with replacement from the validation set. At each step $t$ the sampled agent best responds based on the belief derived from steps $1,...,t-1$. The agent observes the decision from the deployed classifier and then reports the point-decision pair so the belief can be updated. Evaluation checkpoints occur at $m \in \{0,1,2,5,10,20,50,100,200,500,1000,2000\}$. At each checkpoint the belief is frozen and the Agents are scored on the test set using this belief. Crucially, the evaluation on the test set is not used to update the belief.

In the experiment plots, lines corresponding to deterministic trees have an associated ``Transparency Limit''. This gives the accuracy that the Agents would achieve if the classifier was disclosed to them. Note that this value may be different from the $\hat{f}$ value, as the classifier is selected based on the awareness the Agent possesses. In a fully identifiable system, the transparency limit is the asymptotic value belief convergence should approach as $m \to \infty$.

The reported results for these experiments are the average over $20$ random seeds. The order in which Agents are sampled from the stream and, in the case of randomised policies, classifiers are sampled from the mixture vary with the different seeds. The policy used in the experiment is itself fixed before any seed runs, so seeds are independent to each other. 

\keypoint{Ambiguity} In the Ambiguity experiments the Agents maintain a credible set of classifiers that they believe may plausibly have been deployed by the Learner. In our experiments we defined this set, after $m$ observations, as $K_{m}=\{h: w_{h}\geq \tau\}$ where $w_{h}$ is the belief weight assigned to tree $h$, for $\tau=0.05$ and the number of elements in the credible set initially, $k=2$.

Under the probing condition, Agents probe with probability $0.05$. Probing behaviour is realised by computing the current credible set, $K_{m}$, and, if $|K_{m}|\geq 2$, moves to the cheapest point at which some subset of the elements of $K_{m}$ disagree, which is identified by exhaustive evaluation of the leaf-box intersections of each of the classifiers in $K_{m}$. Since these Agents, by assumption, do not care about the actions being ``affordable'', we do not put any utility constraint on probing behaviour. The observations from this behaviour are fed back into the belief mechanism as detailed under Disclosure Experiments.

\keypoint{Best Responses} Best responses on decision trees can be computed exactly by considering the union over axis-aligned box intersections. The cheapest gameable point in a box is identified as a per-dimension clipping, with a tiny ($1e^{-6}$) nudge into the box's interior to ensure the point receives positive classification. An Agent picks the utility-maximising subset according to Equation~\ref{eqn:randomised_best_response}.

\keypoint{Compute and Code} Experiments on synthetic data, the disclosure experiments (Section~\ref{sec:robustness_to_uncertainty}) and the ambiguity experiments (Appendix~\ref{app:robustness_to_uncertainty_ambiguity}) were run on a laptop (Apple M5 Pro, $24$ GB RAM) and each completed in under an hour. The exhaustive evaluations, comprising the pool-configuration search (Section~\ref{sec:setup}), the sampled mixture search, and the $k$-component sweep (Section~\ref{sec:more_than_two_components}), were parallelised over candidate mixtures and run on a CPU server with two AMD EPYC 7713 64-core processors (128 cores) and  1 TB RAM. The most expensive of these, the $k$-component sweep, took approximately 12 hours. No GPUs were used. Evaluating a mixture requires best responses for every Agent to every subset of its components, so cost grows with the number of candidate mixtures ($\binom{N'}{k}$ in Section~\ref{sec:more_than_two_components}) and with $2^{k}$.

Code to reproduce all experiments, including configuration files and random seeds, will be released publicly upon publication. Experiments use Python $3.12$, and PRAXIS $0.0.30$.

\section{Robustness to Uncertainty: Supporting Figures}
\subsection{GMSC and Houses Disclosure Experiment Results}
\label{app:robustness_to_uncertainty_gmsc_and_houses}
\begin{figure}[t]
    \centering
    \begin{subfigure}[t]{0.5\textwidth}
        \centering
        \includegraphics[width=\textwidth]{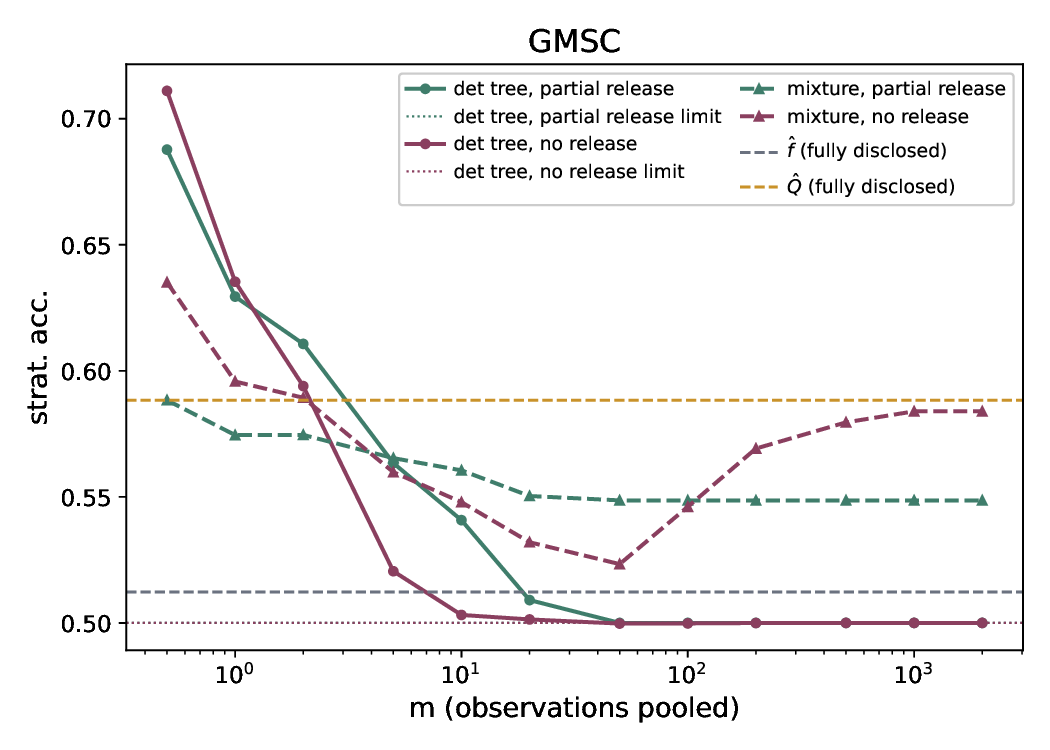}
    \end{subfigure}%
    ~ 
    \begin{subfigure}[t]{0.5\textwidth}
        \centering
        \includegraphics[width=\textwidth]{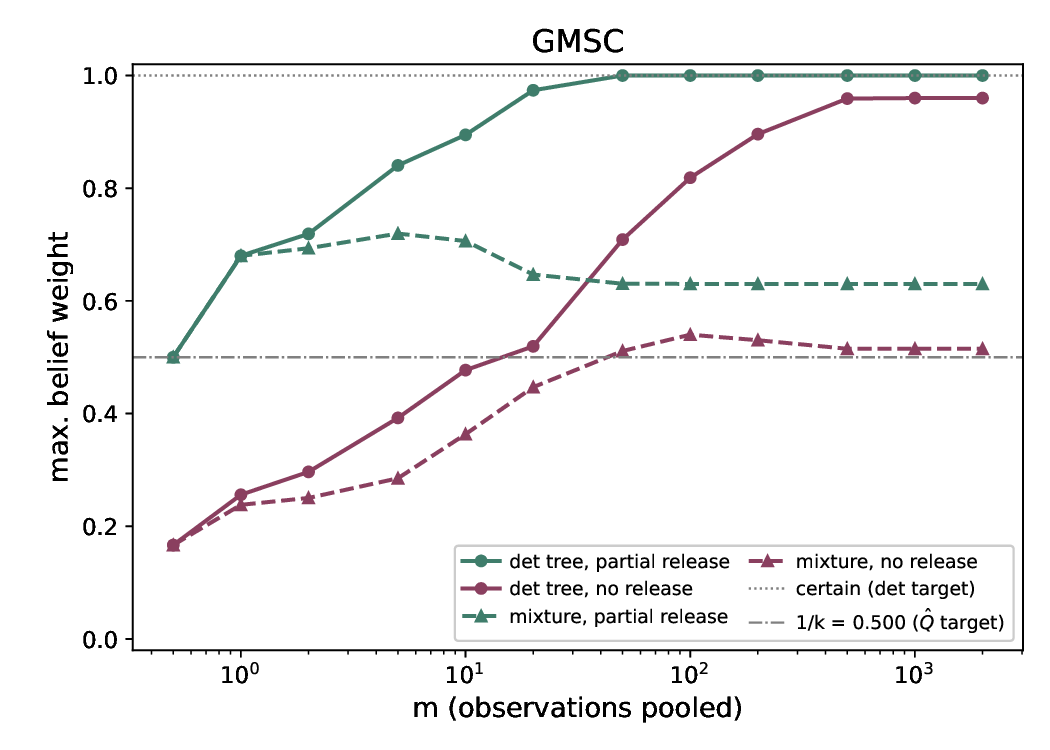}
    \end{subfigure}
    \caption{Examining the change in strategic accuracy and certainty with disclosure on the GMSC dataset; (Left) Comparing the strategic accuracy of the classifier in the case of partial classifier release (\cite{cohen2024}), and no release (\cite{ghalme2021}) to $\hat{Q}$ (dashed yellow line) and $\hat{f}$ (dashed grey line) as the number of Agent observations ($m$) increases. Dotted lines give the strategic accuracy achieved if Agents know classifier exactly; (Right) Tracking how Agent belief in classifier identity changes with the number of observations}
    \label{fig:robustness_to_uncertainty_gmsc}
\end{figure}

\begin{figure}[t]
    \centering
    \begin{subfigure}[t]{0.5\textwidth}
        \centering
        \includegraphics[width=\textwidth]{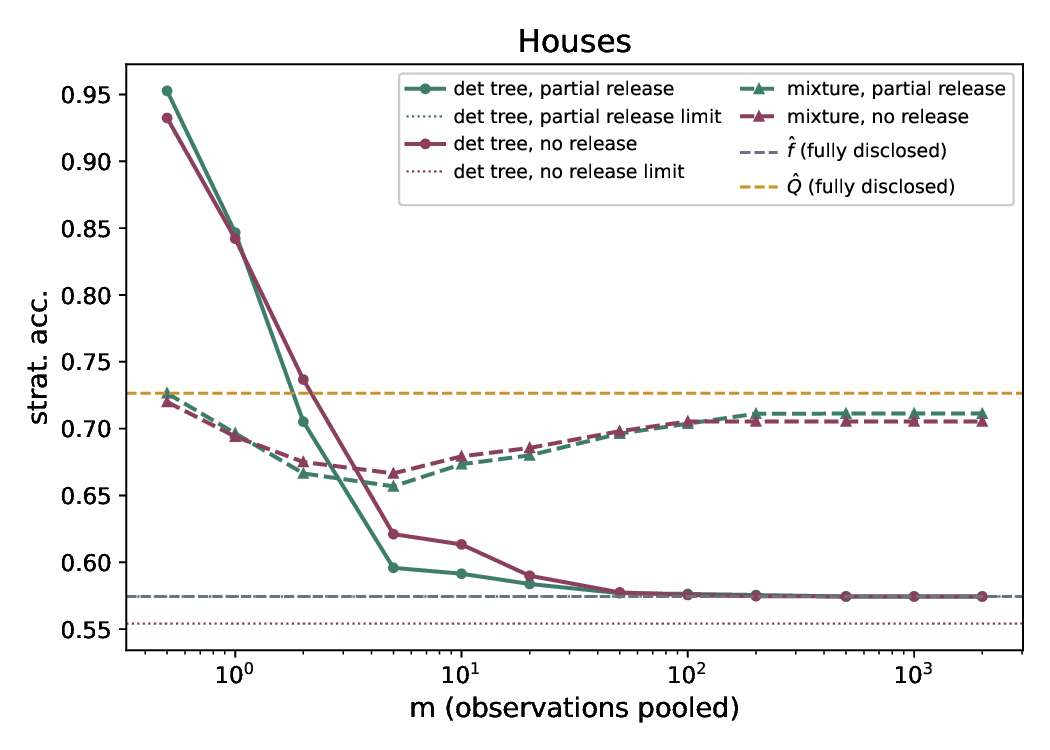}
    \end{subfigure}%
    ~ 
    \begin{subfigure}[t]{0.5\textwidth}
        \centering
        \includegraphics[width=\textwidth]{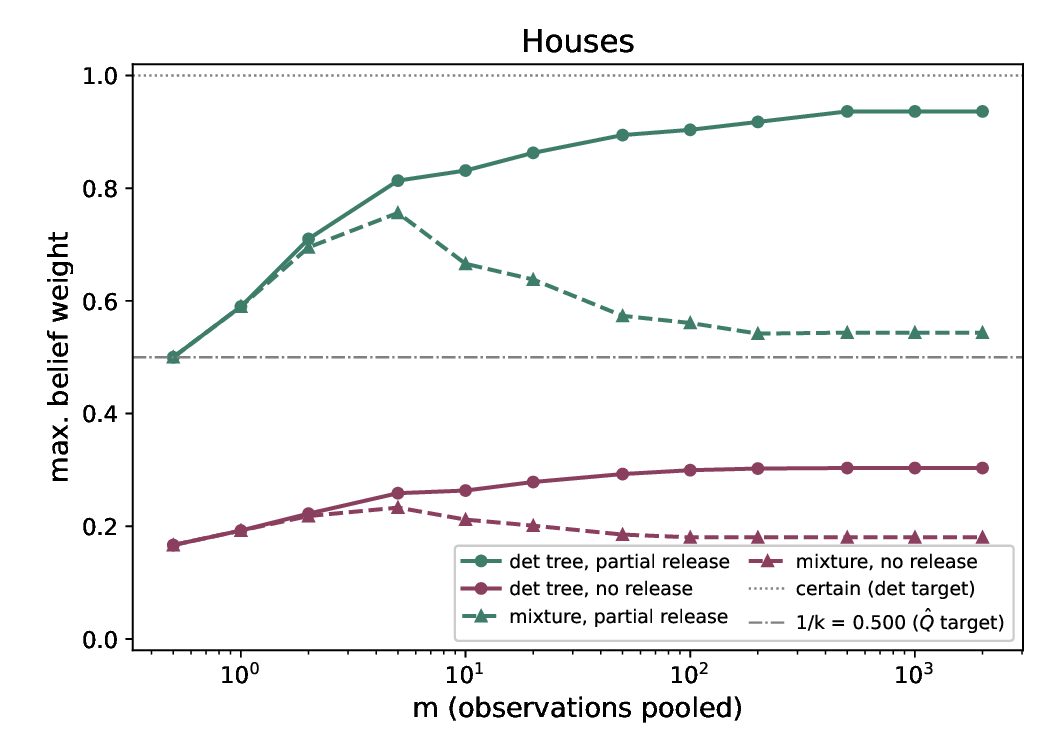}
    \end{subfigure}
    \caption{Examining the change in strategic accuracy and certainty with disclosure on the Houses dataset; (Left) Comparing the strategic accuracy of the classifier in the case of partial classifier release (\cite{cohen2024}), and no release (\cite{ghalme2021}) to $\hat{Q}$ (dashed yellow line) and $\hat{f}$ (dashed grey line) as the number of Agent observations ($m$) increases. Dotted lines give the strategic accuracy achieved if Agents know classifier exactly; (Right) Tracking how Agent belief in classifier identity changes with the number of observations}
    \label{fig:robustness_to_uncertainty_houses}
\end{figure}

Figure~\ref{fig:robustness_to_uncertainty_gmsc} and Figure~\ref{fig:robustness_to_uncertainty_houses} present the results from the experiment detailed in Section~\ref{sec:robustness_to_uncertainty} on the GMSC and Houses datasets respectively. These results support the previously reported results closely; all of the disclosure approaches initially approximately match or exceed $\hat{Q}$, the optimal randomised classifier with full transparency. However, as the Agents aggregate observations, the accuracy in each policy falls to beneath the $\hat{Q}$ performance. 

On both datasets the deterministic classifiers, regardless of disclosure method, converge to their limit behaviour. Looking at the beliefs, unlike in the Credit Card Fraud case, the Agent belief plateaus below certainty; in the Houses dataset under no release, the max belief over the candidate set is $\sim 0.3$. This indicates that, on this dataset, some classifiers are indistinguishable to the agents. This also affects the randomised setting, preventing these approaches from exactly converging to the $\hat{Q}$ behaviour. 

\subsection{Ambiguity Results}
\label{app:robustness_to_uncertainty_ambiguity}
\begin{figure}[t]
    \centering
    \begin{subfigure}[t]{0.5\textwidth}
        \centering
        \includegraphics[width=\textwidth]{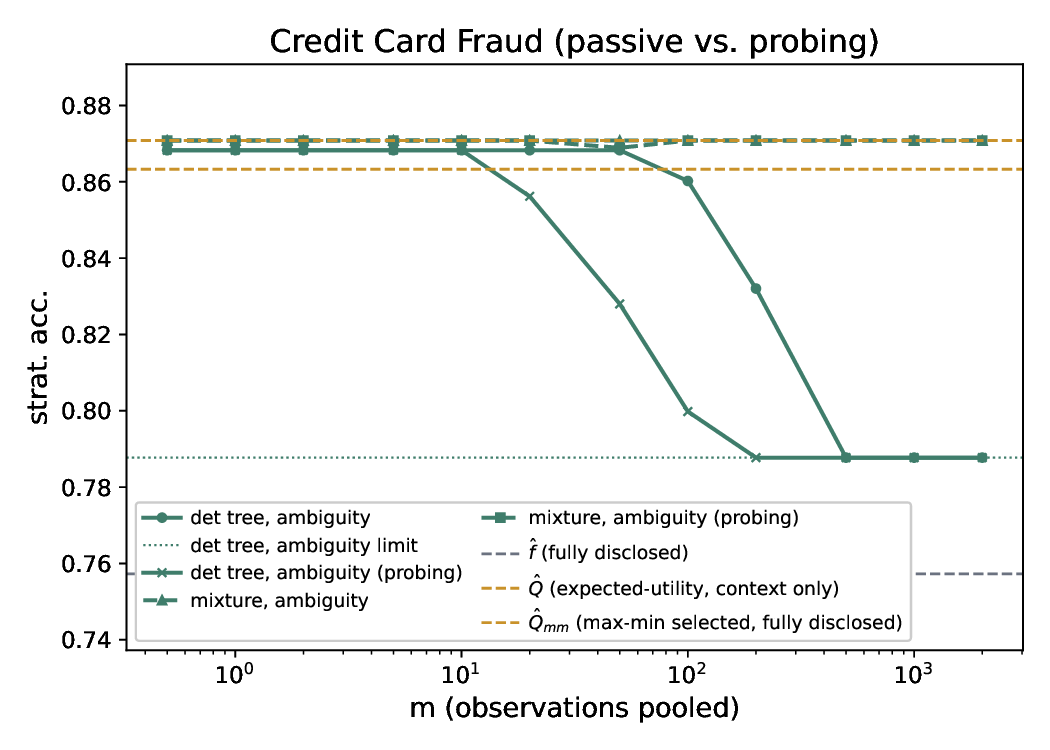}
    \end{subfigure}%
    ~ 
    \begin{subfigure}[t]{0.5\textwidth}
        \centering
        \includegraphics[width=\textwidth]{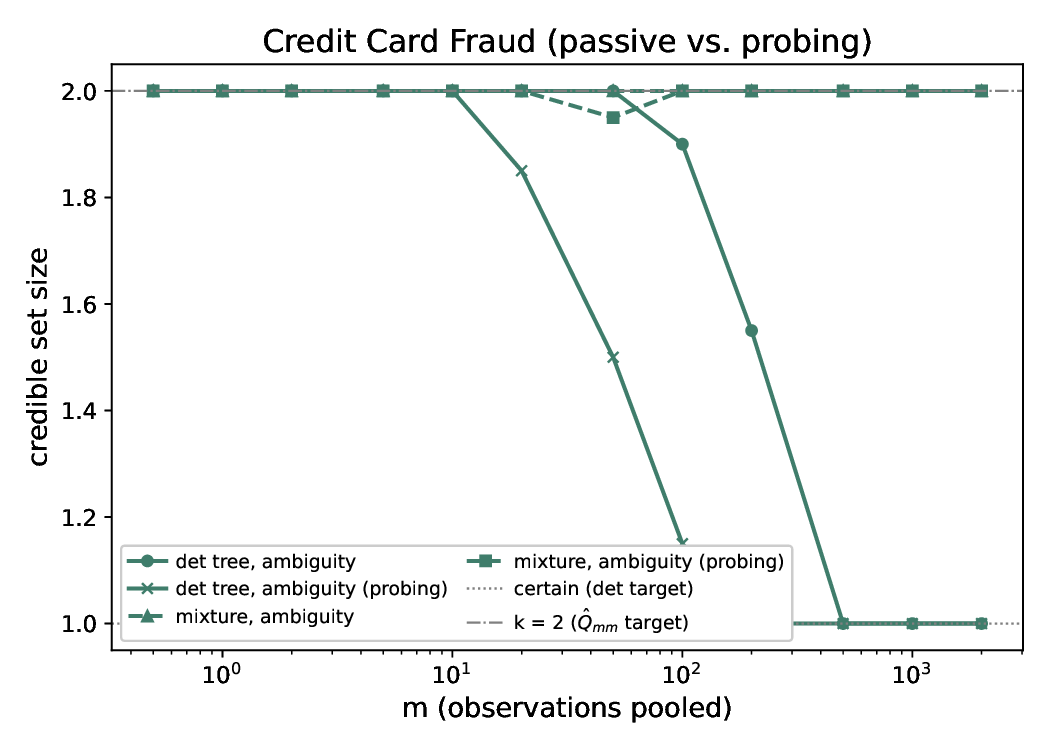}
    \end{subfigure}
    \caption{Examining the change in strategic accuracy and credible set size in the Ambiguity setting on the Credit Card Fraud dataset; (Left) Comparing the strategic accuracy of the classifier under passive observation and active probing behaviour to the optimal randomised classifier (yellow dashed line) and the optimal deterministic classifier (grey dashed line) as the number of Agent observations ($m$) increases. (Right) Tracking how Agent credible set size changes with the number of observations.}
    \label{fig:robustness_to_ambiguity_credit_card_fraud}
\end{figure}

\begin{figure}[t]
    \centering
    \begin{subfigure}[t]{0.5\textwidth}
        \centering
        \includegraphics[width=\textwidth]{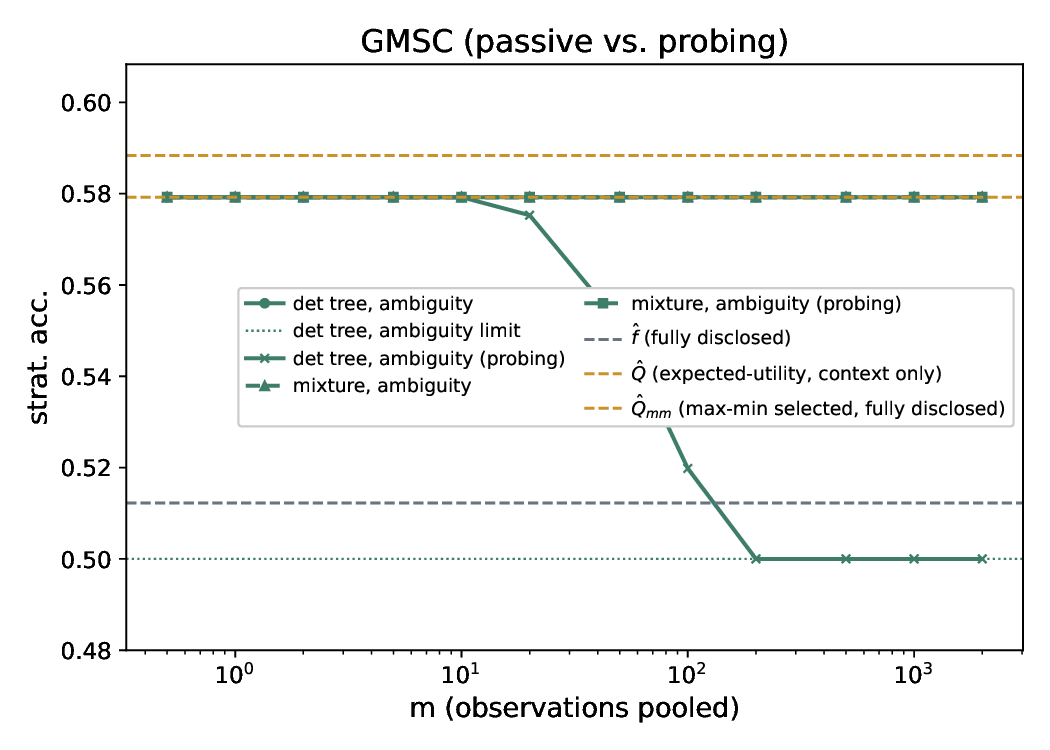}
    \end{subfigure}%
    ~ 
    \begin{subfigure}[t]{0.5\textwidth}
        \centering
        \includegraphics[width=\textwidth]{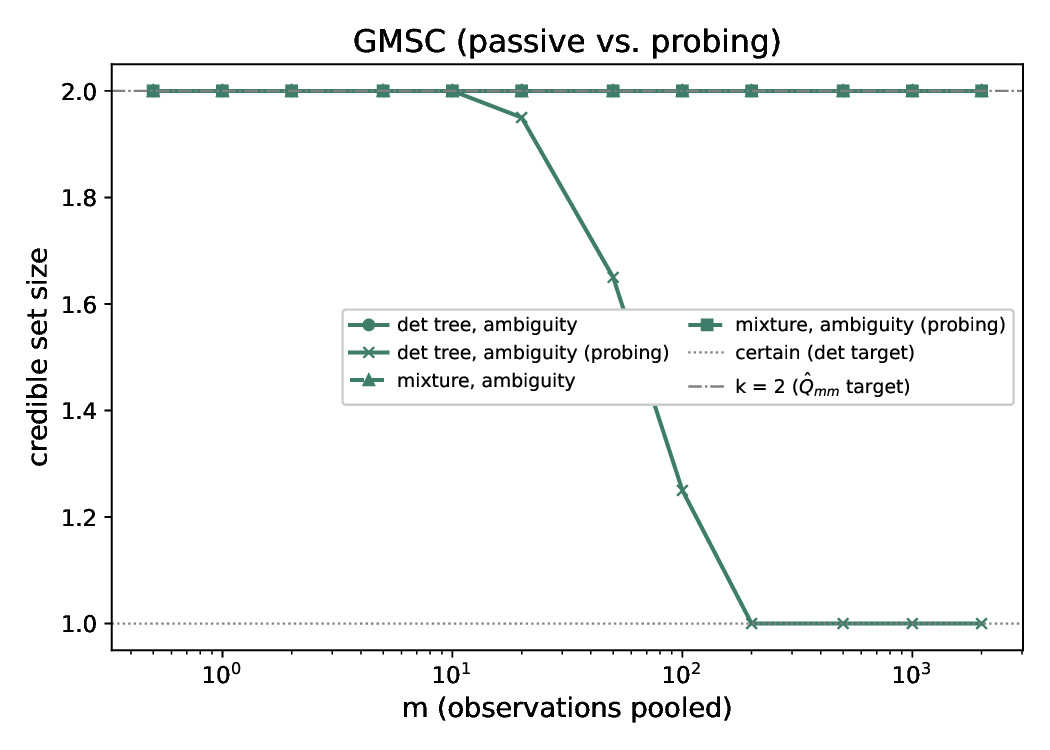}
    \end{subfigure}
    \caption{Examining the change in strategic accuracy and credible set size in the Ambiguity setting on the GMSC dataset; (Left) Comparing the strategic accuracy of the classifier under passive observation and active probing behaviour to the optimal randomised classifier (yellow dashed line) and the optimal deterministic classifier (grey dashed line) as the number of Agent observations ($m$) increases. (Right) Tracking how Agent credible set size changes with the number of observations.}
    \label{fig:robustness_to_ambiguity_gmsc}
\end{figure}

\begin{figure}[t]
    \centering
    \begin{subfigure}[t]{0.5\textwidth}
        \centering
        \includegraphics[width=\textwidth]{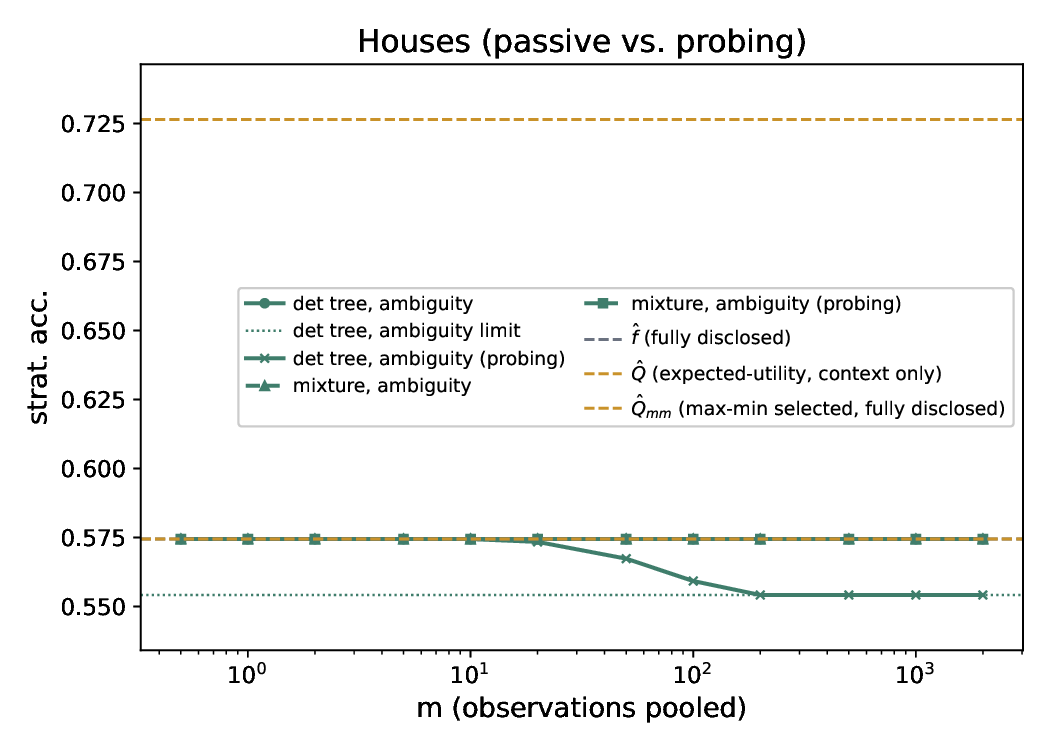}
    \end{subfigure}%
    ~ 
    \begin{subfigure}[t]{0.5\textwidth}
        \centering
        \includegraphics[width=\textwidth]{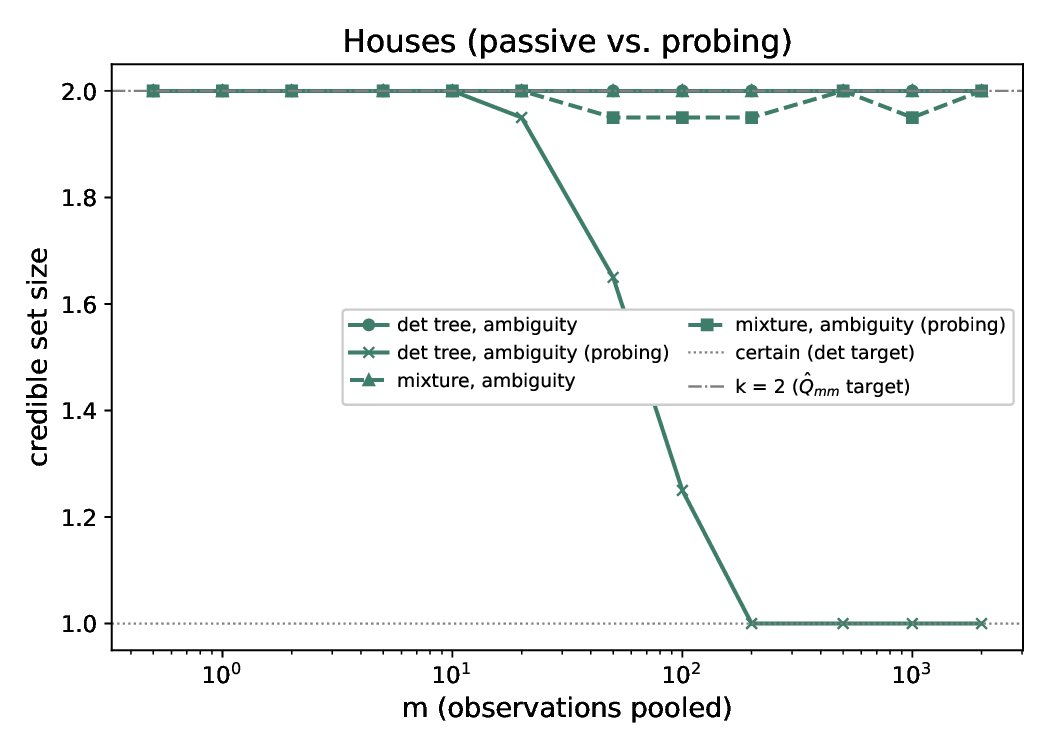}
    \end{subfigure}
    \caption{Examining the change in strategic accuracy and credible set size in the Ambiguity setting on the Houses dataset; (Left) Comparing the strategic accuracy of the classifier under passive observation and active probing behaviour to the optimal randomised classifier (yellow dashed line) and the optimal deterministic classifier (grey dashed line) as the number of Agent observations ($m$) increases. (Right) Tracking how Agent credible set size changes with the number of observations.}
    \label{fig:robustness_to_ambiguity_houses}
\end{figure}

\cite{hikri2026} propose a variant of the conventional strategic classification problem where the Learner deploys a fixed classifier, and discloses a set of possible classifiers (the ``ambiguity set'', $\Gamma$) containing the deployed classifier to the Agent. It is further understood that Agents in this setting will aim to game all the disclosed classifiers, or else will not game at all (max-min decision-making). This is in contrast with the expected utility maximisation objective used in this work and in the literature \citep{cohen2024, ghalme2021}.

For the ambiguity set to be meaningful it must contain classifiers that induce distinct gaming behaviours. For our experiments we achieve this by using the same procedure used in Section~\ref{sec:more_than_two_components}; we cluster classifiers with identical behaviours and sample a single candidate classifier per cluster. We initialise the candidate pool, $C$, with the components of $\hat{Q}$, the best sampled mixture. This encourages diversity and facilitates comparison with this baseline. Greedy search is then used to iteratively add candidates to the candidate pool, where candidates are scored by the largest minimum pairwise train disagreement with current members of $C$ under $|C|=6$. The deployed fixed tree and uniform mixture classifiers, as well as the corresponding ambiguity sets, are then identified by exhaustive evaluation over $C$, using the max-min objective to model Agent behaviour. To ensure that the results of the different methods are comparable we let $|\Gamma|=k=2$ so that the size of the ambiguity set is the same as the support of the mixture. Otherwise the experiment procedure is identical to the setup proposed in Section~\ref{sec:robustness_to_uncertainty}.

Figure~\ref{fig:robustness_to_ambiguity_credit_card_fraud}, Figure~\ref{fig:robustness_to_ambiguity_gmsc}, and Figure~\ref{fig:robustness_to_ambiguity_houses} present the results of this experiment on the Credit Card Fraud, GMSC, and Houses datasets respectively. Left panels compare strategic accuracies under the various disclosure policies as the number of observations aggregated by the Agents, $m$, increases. Since Agents in the max-min setting do not let belief inform their response, we instead assume that belief affects how ``credible'' each classifier is. A classifier with belief weight less than $\tau=0.05$ is considered incredible, and therefore the Agents eject it from the credible set. The right panels in the figures depict the changes in the size of the credible set as $m$ increases.

The max-min condition on Agent behaviour undermines the Agents' ability to do classifier inference; since Agents will only game if they can game every classifier, and instead will do nothing at all, for reasonably accurate classifiers the Agents will not be able to distinguish between the classifiers. To address this, with probability $0.05$ an Agent will ``sacrifice'' themselves and choose their behaviour to probe the credible set. This is achieved by having the Agent choose an action that would maximise the disagreement between the classifiers in the credible set, thereby providing meaningful signal from their observation. 

From the figures it is clear that the fixed tree and mixture classifiers initially achieve identical accuracies as no Agents remains where the candidates disagree. In the Credit Card Fraud case, the Agents are eventually able to determine the deployed classifier, at which point the accuracy collapses to the transparency limit (as in the original experiment). On the remaining datasets, passive observation is insufficient to identify the deployed classifier, and Agents remain uncertain with no change in the credible set. Once probing is introduced, all of the deterministic classifier settings see their effectiveness collapse down to the limit. On Credit Card Fraud, observe that the probing allows the Agents to identify the classifier with fewer observations than in the passive case. Supporting our hypothesis, the mixture setting retains its advantage, demonstrating robustness to both passive and probing observation behaviours.

\end{document}